\documentclass{article}

\PassOptionsToPackage{numbers, compress}{natbib}

\usepackage[final]{neurips_2023}

\usepackage{microtype}
\usepackage{graphicx}
\usepackage{subfigure}
\usepackage{booktabs} 
\usepackage{float}

\usepackage{arydshln}

\usepackage{amsmath}
\usepackage{amssymb}
\usepackage{mathtools}
\usepackage{amsthm}
\usepackage[pagebackref=true,breaklinks=true,letterpaper=true,colorlinks,bookmarks=false]{hyperref}

\usepackage[capitalize,noabbrev]{cleveref}

\theoremstyle{plain}
\newtheorem{theorem}{Theorem}[section]

\newtheorem{lemma}[theorem]{Lemma}

\theoremstyle{definition}
\newtheorem{definition}[theorem]{Definition}

\theoremstyle{remark}

\makeatletter
\newcommand*\bigcdot{\mathpalette\bigcdot@{.5}}
\newcommand*\bigcdot@[2]{\mathbin{\vcenter{\hbox{\scalebox{#2}{$\m@th#1\bullet$}}}}}
\makeatother

\usepackage[textsize=tiny]{todonotes}

\title{Flat Seeking Bayesian Neural Networks}

\author{%
Van-Anh Nguyen\textsuperscript{\rm 1}
\and
\textbf{Tung-Long Vuong}\textsuperscript{\rm 1,}\textsuperscript{\rm 2}
\and
\textbf{Hoang Phan}\textsuperscript{\rm 2,}\textsuperscript{\rm 3}
\and
\textbf{Thanh-Toan Do}\textsuperscript{\rm 1}
\and
\textbf{Dinh Phung} \textsuperscript{\rm 1,}\textsuperscript{\rm 2} ~~~~~~~~~
\textbf{Trung Le} \textsuperscript{\rm 1}
\\
\textsuperscript{\rm 1}Department of Data Science and AI, Monash University, Australia \\
\textsuperscript{\rm 2}VinAI, Vietnam \\
\textsuperscript{\rm 3}New York University, United States \\
\{\tt\small van-anh.nguyen, tung-long.vuong, toan.do, dinh.phung, trunglm\}@monash.edu \\
{\tt\small hvp2011@nyu.edu} ~
}

\begin{document}
\maketitle



\newcommand{\sidenote}[1]{\marginpar{\small \emph{\color{Medium}#1}}}

\global\long\def\se{\hat{\text{se}}}%

\global\long\def\interior{\text{int}}%

\global\long\def\boundary{\text{bd}}%

\global\long\def\ML{\textsf{ML}}%

\global\long\def\GML{\mathsf{GML}}%

\global\long\def\HMM{\mathsf{HMM}}%

\global\long\def\support{\text{supp}}%

\global\long\def\new{\text{*}}%

\global\long\def\stir{\text{Stirl}}%

\global\long\def\mA{\mathcal{A}}%

\global\long\def\mB{\mathcal{B}}%

\global\long\def\mF{\mathcal{F}}%

\global\long\def\mK{\mathcal{K}}%

\global\long\def\mH{\mathcal{H}}%

\global\long\def\mX{\mathcal{X}}%

\global\long\def\mZ{\mathcal{Z}}%

\global\long\def\mS{\mathcal{S}}%

\global\long\def\Ical{\mathcal{I}}%

\global\long\def\mT{\mathcal{T}}%

\global\long\def\Pcal{\mathcal{P}}%

\global\long\def\dist{d}%

\global\long\def\HX{\entro\left(X\right)}%
 
\global\long\def\entropyX{\HX}%

\global\long\def\HY{\entro\left(Y\right)}%
 
\global\long\def\entropyY{\HY}%

\global\long\def\HXY{\entro\left(X,Y\right)}%
 
\global\long\def\entropyXY{\HXY}%

\global\long\def\mutualXY{\mutual\left(X;Y\right)}%
 
\global\long\def\mutinfoXY{\mutualXY}%

\global\long\def\given{\mid}%

\global\long\def\gv{\given}%

\global\long\def\goto{\rightarrow}%

\global\long\def\asgoto{\stackrel{a.s.}{\longrightarrow}}%

\global\long\def\pgoto{\stackrel{p}{\longrightarrow}}%

\global\long\def\dgoto{\stackrel{d}{\longrightarrow}}%

\global\long\def\lik{\mathcal{L}}%

\global\long\def\logll{\mathit{l}}%

\global\long\def\vectorize#1{\mathbf{#1}}%

\global\long\def\vt#1{\mathbf{#1}}%

\global\long\def\gvt#1{\boldsymbol{#1}}%

\global\long\def\idp{\ \bot\negthickspace\negthickspace\bot\ }%
 
\global\long\def\cdp{\idp}%

\global\long\def\das{}%

\global\long\def\id{\mathbb{I}}%

\global\long\def\idarg#1#2{\id\left\{  #1,#2\right\}  }%

\global\long\def\iid{\stackrel{\text{iid}}{\sim}}%

\global\long\def\bzero{\vt 0}%

\global\long\def\bone{\mathbf{1}}%

\global\long\def\boldm{\boldsymbol{m}}%

\global\long\def\bff{\vt f}%

\global\long\def\bx{\boldsymbol{x}}%

\global\long\def\bd{\boldsymbol{d}}%

\global\long\def\bl{\boldsymbol{l}}%

\global\long\def\bu{\boldsymbol{u}}%

\global\long\def\bo{\boldsymbol{o}}%

\global\long\def\bh{\boldsymbol{h}}%

\global\long\def\bs{\boldsymbol{s}}%

\global\long\def\bz{\boldsymbol{z}}%

\global\long\def\xnew{y}%

\global\long\def\bxnew{\boldsymbol{y}}%

\global\long\def\bX{\boldsymbol{X}}%

\global\long\def\tbx{\tilde{\bx}}%

\global\long\def\by{\boldsymbol{y}}%

\global\long\def\bY{\boldsymbol{Y}}%

\global\long\def\bZ{\boldsymbol{Z}}%

\global\long\def\bU{\boldsymbol{U}}%

\global\long\def\bv{\boldsymbol{v}}%

\global\long\def\bn{\boldsymbol{n}}%

\global\long\def\bV{\boldsymbol{V}}%

\global\long\def\bI{\boldsymbol{I}}%

\global\long\def\bw{\vt w}%

\global\long\def\balpha{\gvt{\alpha}}%

\global\long\def\bbeta{\gvt{\beta}}%

\global\long\def\bmu{\gvt{\mu}}%

\global\long\def\btheta{\boldsymbol{\theta}}%

\global\long\def\blambda{\boldsymbol{\lambda}}%

\global\long\def\bgamma{\boldsymbol{\gamma}}%

\global\long\def\bpsi{\boldsymbol{\psi}}%

\global\long\def\bphi{\boldsymbol{\phi}}%

\global\long\def\bpi{\boldsymbol{\pi}}%

\global\long\def\bomega{\boldsymbol{\omega}}%

\global\long\def\bepsilon{\boldsymbol{\epsilon}}%

\global\long\def\btau{\boldsymbol{\tau}}%

\global\long\def\bxi{\boldsymbol{\xi}}%

\global\long\def\realset{\mathbb{R}}%

\global\long\def\realn{\realset^{n}}%

\global\long\def\integerset{\mathbb{Z}}%

\global\long\def\natset{\integerset}%

\global\long\def\integer{\integerset}%

\global\long\def\natn{\natset^{n}}%

\global\long\def\rational{\mathbb{Q}}%

\global\long\def\rationaln{\rational^{n}}%

\global\long\def\complexset{\mathbb{C}}%

\global\long\def\comp{\complexset}%

\global\long\def\compl#1{#1^{\text{c}}}%

\global\long\def\and{\cap}%

\global\long\def\compn{\comp^{n}}%

\global\long\def\comb#1#2{\left({#1\atop #2}\right) }%

\global\long\def\nchoosek#1#2{\left({#1\atop #2}\right)}%

\global\long\def\param{\vt w}%

\global\long\def\Param{\Theta}%

\global\long\def\meanparam{\gvt{\mu}}%

\global\long\def\Meanparam{\mathcal{M}}%

\global\long\def\meanmap{\mathbf{m}}%

\global\long\def\logpart{A}%

\global\long\def\simplex{\Delta}%

\global\long\def\simplexn{\simplex^{n}}%

\global\long\def\dirproc{\text{DP}}%

\global\long\def\ggproc{\text{GG}}%

\global\long\def\DP{\text{DP}}%

\global\long\def\ndp{\text{nDP}}%

\global\long\def\hdp{\text{HDP}}%

\global\long\def\gempdf{\text{GEM}}%

\global\long\def\rfs{\text{RFS}}%

\global\long\def\bernrfs{\text{BernoulliRFS}}%

\global\long\def\poissrfs{\text{PoissonRFS}}%

\global\long\def\grad{\gradient}%
 
\global\long\def\gradient{\nabla}%

\global\long\def\partdev#1#2{\partialdev{#1}{#2}}%
 
\global\long\def\partialdev#1#2{\frac{\partial#1}{\partial#2}}%

\global\long\def\partddev#1#2{\partialdevdev{#1}{#2}}%
 
\global\long\def\partialdevdev#1#2{\frac{\partial^{2}#1}{\partial#2\partial#2^{\top}}}%

\global\long\def\closure{\text{cl}}%

\global\long\def\cpr#1#2{\Pr\left(#1\ |\ #2\right)}%

\global\long\def\var{\text{Var}}%

\global\long\def\Var#1{\text{Var}\left[#1\right]}%

\global\long\def\cov{\text{Cov}}%

\global\long\def\Cov#1{\cov\left[ #1 \right]}%

\global\long\def\COV#1#2{\underset{#2}{\cov}\left[ #1 \right]}%

\global\long\def\corr{\text{Corr}}%

\global\long\def\sst{\text{T}}%

\global\long\def\SST{\sst}%

\global\long\def\ess{\mathbb{E}}%

\global\long\def\Ess#1{\ess\left[#1\right]}%


\global\long\def\fisher{\mathcal{F}}%

\global\long\def\bfield{\mathcal{B}}%
 
\global\long\def\borel{\mathcal{B}}%

\global\long\def\bernpdf{\text{Bernoulli}}%

\global\long\def\betapdf{\text{Beta}}%

\global\long\def\dirpdf{\text{Dir}}%

\global\long\def\gammapdf{\text{Gamma}}%

\global\long\def\gaussden#1#2{\text{Normal}\left(#1, #2 \right) }%

\global\long\def\gauss{\mathbf{N}}%

\global\long\def\gausspdf#1#2#3{\text{Normal}\left( #1 \lcabra{#2, #3}\right) }%

\global\long\def\multpdf{\text{Mult}}%

\global\long\def\poiss{\text{Pois}}%

\global\long\def\poissonpdf{\text{Poisson}}%

\global\long\def\pgpdf{\text{PG}}%

\global\long\def\wshpdf{\text{Wish}}%

\global\long\def\iwshpdf{\text{InvWish}}%

\global\long\def\nwpdf{\text{NW}}%

\global\long\def\niwpdf{\text{NIW}}%

\global\long\def\studentpdf{\text{Student}}%

\global\long\def\unipdf{\text{Uni}}%

\global\long\def\transp#1{\transpose{#1}}%
 
\global\long\def\transpose#1{#1^{\mathsf{T}}}%

\global\long\def\mgt{\succ}%

\global\long\def\mge{\succeq}%

\global\long\def\idenmat{\mathbf{I}}%

\global\long\def\trace{\mathrm{tr}}%

\global\long\def\argmax#1{\underset{_{#1}}{\text{argmax}} }%

\global\long\def\argmin#1{\underset{_{#1}}{\text{argmin}\ } }%

\global\long\def\diag{\text{diag}}%

\global\long\def\norm{}%

\global\long\def\spn{\text{span}}%

\global\long\def\vtspace{\mathcal{V}}%

\global\long\def\field{\mathcal{F}}%
 
\global\long\def\ffield{\mathcal{F}}%

\global\long\def\inner#1#2{\left\langle #1,#2\right\rangle }%
 
\global\long\def\iprod#1#2{\inner{#1}{#2}}%

\global\long\def\dprod#1#2{#1 \cdot#2}%

\global\long\def\norm#1{\left\Vert #1\right\Vert }%

\global\long\def\entro{\mathbb{H}}%

\global\long\def\entropy{\mathbb{H}}%

\global\long\def\Entro#1{\entro\left[#1\right]}%

\global\long\def\Entropy#1{\Entro{#1}}%

\global\long\def\mutinfo{\mathbb{I}}%

\global\long\def\relH{\mathit{D}}%

\global\long\def\reldiv#1#2{\relH\left(#1||#2\right)}%

\global\long\def\KL{KL}%

\global\long\def\KLdiv#1#2{\KL\left(#1\parallel#2\right)}%
 
\global\long\def\KLdivergence#1#2{\KL\left(#1\ \parallel\ #2\right)}%

\global\long\def\crossH{\mathcal{C}}%
 
\global\long\def\crossentropy{\mathcal{C}}%

\global\long\def\crossHxy#1#2{\crossentropy\left(#1\parallel#2\right)}%

\global\long\def\breg{\text{BD}}%

\global\long\def\lcabra#1{\left|#1\right.}%

\global\long\def\lbra#1{\lcabra{#1}}%

\global\long\def\rcabra#1{\left.#1\right|}%

\global\long\def\rbra#1{\rcabra{#1}}%

\begin{abstract}
Bayesian Neural Networks (BNNs) provide a probabilistic interpretation for deep learning models by imposing a prior distribution over model parameters and inferring a posterior distribution based on observed data. The model sampled from the posterior distribution can be used for providing ensemble predictions and quantifying prediction uncertainty. It is well-known that deep learning models with lower sharpness have better generalization ability. However, existing posterior inferences are not aware of sharpness/flatness in terms of formulation, possibly leading to high sharpness for the models sampled from them. In this paper, we develop theories, the Bayesian setting, and the variational inference approach for the sharpness-aware posterior. Specifically, the models sampled from our sharpness-aware posterior, and the optimal approximate posterior estimating this sharpness-aware posterior, have better flatness, hence possibly possessing higher generalization ability. We conduct experiments by leveraging the sharpness-aware posterior with state-of-the-art Bayesian Neural Networks, showing that the flat-seeking counterparts outperform their baselines in all metrics of interest.

\end{abstract}

\section{Introduction}
\label{sec:intro}
Bayesian Neural Networks (BNNs) provide a way to interpret deep learning models probabilistically. This is done by setting a prior distribution over model parameters and then inferring a posterior distribution over model parameters based on observed data. This allows us to not only make predictions, but also quantify prediction uncertainty, which is useful for many real-world applications. To sample deep learning models from complex and complicated posterior distributions, advanced particle-sampling approaches such as Hamiltonian Monte Carlo (HMC) \cite{Neal_Bayesian}, Stochastic Gradient HMC (SGHMC) \cite{chen2014stochastic}, Stochastic Gradient Langevin dynamics (SGLD) \cite{welling2011bayesian}, and Stein Variational Gradient Descent (SVGD) \cite{liu2016stein} are often used. However, these methods can be computationally expensive, particularly when many models need to be sampled for better ensembles.


To alleviate this computational burden and enable the sampling of multiple deep learning models from posterior distributions, variational inference approaches employ approximate posteriors to estimate the true posterior. These methods utilize approximate posteriors that belong to sufficiently rich families, which are both economical and convenient to sample from. However, the pioneering works in variational inference, such as \cite{graves2011practical, blundell2015weight, kingma2015variational}, assume approximate posteriors to be fully factorized distributions, also known as mean-field variational inference. This approach fails to account for the strong statistical dependencies among random weights of neural networks, limiting its ability to capture the complex structure of the true posterior and estimate the true model uncertainty. To overcome this issue, latter works have attempted to provide posterior approximations with richer expressiveness \cite{Zhang2018DeepML, ritter2018scalable, rossi2020walsh, swiatkowski2020k, ghosh2018structured, ong2018gaussian, tomczak2020efficient, khan2018fast, cuong2023}. These approaches aim to improve the accuracy of the posterior approximation and enable more effective uncertainty quantification.

In the context of standard deep network training, it has been observed that flat minimizers can enhance the generalization capability of models. This is achieved by enabling them to locate wider local minima that are more robust to shifts between train and test sets. Several studies, including \cite{DBLP:conf/iclr/JiangNMKB20, DBLP:conf/nips/PetzkaKASB21, DBLP:conf/uai/DziugaiteR17}, have shown evidence to support this principle.
However, the posteriors used in existing Bayesian neural networks (BNNs) do not account for the sharpness/flatness of the models derived from them in terms of model formulation. As a result, the sampled models can be located in regions of high sharpness and low flatness, leading to poor generalization ability. Moreover, in variational inference methods, using approximate posteriors to estimate these non-sharpness-aware posteriors can result in sampled models from the corresponding optimal approximate posterior lacking awareness of sharpness/flatness, hence causing them to suffer from poor generalization ability.
In this paper, our objective is to propose a sharpness-aware posterior for learning BNNs, which samples models with high flatness for better generalization ability. To achieve this, we devise both a Bayesian setting and a variational inference approach for the proposed posterior. By estimating the optimal approximate posteriors, we can generate flatter models that improve the generalization ability.
Our approach is as follows: In Theorem \ref{thm:Gibbs_posterior}, we show that the standard posterior is the optimal solution to an optimization problem that balances the empirical loss induced by models sampled from an approximate posterior for fitting a training set with a Kullback-Leibler (KL) divergence, which encourages a simple approximate posterior. Based on this insight, we replace the empirical loss induced by the approximate posterior with the general loss over the entire data-label distribution in Theorem \ref{thm:general_loss} to improve the generalization ability.
Inspired by sharpness-aware minimization \cite{foret2021sharpnessaware}, we develop an upper-bound of the general loss in Theorem \ref{thm:general_loss}, leading us to formulate the sharpness-aware posterior in Theorem \ref{thm:Bayses_SAM_solution}. Finally, we devise the Bayesian setting and variational approach for the sharpness-aware posterior. Overall, our contributions in this paper can be summarized as follows:
\begin{itemize}
    \item We propose and develop theories, the Bayesian setting, and the variational inference approach for the sharpness-aware posterior. This posterior enables us to sample a set of flat models that improve the model generalization ability. We note that SAM \cite{foret2021sharpnessaware} only considers the sharpness for a single model, while ours is the first work studying the concept and theory of the sharpness for a distribution $\mathbb{Q}$ over models. Additionally, the proof of Theorem \ref{thm:general_loss} is very challenging, elegant, and complicated because of the infinite number of models in the support of $\mathbb{Q}$.  
    \item We conduct extensive experiments by leveraging our sharpness-aware posterior with the state-of-the-art and well-known BNNs, including \textit{BNNs with an approximate Gaussian distribution \cite{kingma2015variational}}, \textit{BNNs with stochastic gradient Langevin dynamics
(SGLD) \cite{welling2011bayesian}}, \textit{MC-Dropout \cite{gal2016-mcdropout}}, \textit{Bayesian deep ensemble \cite{lakshminarayanan2017simple}}, and \textit{SWAG \cite{Maddox2019-swag}} to demonstrate that the flat-seeking counterparts consistently outperform the corresponding approaches in all metrics of interest, including the ensemble accuracy, expected calibration error (ECE), and negative log-likelihood (NLL). 
    
\end{itemize}


\section{Related Work}
\label{sec:relatedwork}
\subsection{Bayesian Neural Networks}
\textbf{Markov chain Monte Carlo (MCMC):} This approach allows us to sample multiple models from the posterior distribution and was well-known for inference with neural networks through the Hamiltonian Monte Carlo (HMC) \cite{Neal_Bayesian}. However, HMC requires the estimation of full gradients, which is computationally expensive for neural networks. To make the HMC
framework practical, Stochastic Gradient HMC (SGHMC) \cite{chen2014stochastic} enables
stochastic gradients to be used in Bayesian inference, crucial for both scalability and exploring a space of solutions. Alternatively, stochastic gradient Langevin dynamics
(SGLD) \cite{welling2011bayesian} employs first-order Langevin dynamics in the stochastic gradient setting. Additionally, Stein Variational Gradient Descent (SVGD) \cite{liu2016stein} maintains a set of particles to gradually approach a posterior distribution. Theoretically,
all SGHMC, SGLD, and SVGD asymptotically sample from the posterior in the limit of infinitely small
step sizes. 

\textbf{Variational Inference}: This approach uses an approximate posterior distribution in a family to estimate the true posterior distribution by maximizing a variational lower bound. \cite{graves2011practical}  suggests fitting a Gaussian variational posterior approximation over the weights of neural networks, which was generalized in \cite{kingma2013auto, kingma2015variational, blundell2015weight}, using the reparameterization trick for training deep latent variable models. To provide posterior approximations
with richer expressiveness, many extensive studies have been proposed. Notably, \cite{louizos2017multiplicative} treats the weight matrix as a whole via a matrix
variate Gaussian \cite{gupta2018matrix} and approximates the posterior based on this parameterization.  Several later works have inspected this distribution to examine different structured representations for the
variational Gaussian posterior, such as Kronecker-factored \cite{zhang2018noisy, ritter2018scalable, rossi2020walsh}, k-tied distribution \cite{swiatkowski2020k}, non-centered or rank-1 parameterization \cite{ghosh2018structured, dusenberry2020efficient}. Another recipe to represent the true covariance
matrix of Gaussian posterior is through the low-rank approximation \cite{ong2018gaussian, tomczak2020efficient, khan2018fast, Maddox2019-swag}. 

\textbf{Dropout Variational Inference:} This approach utilizes dropout to characterize approximate posteriors. Typically, \cite{gal2016-mcdropout} and \cite{kingma2015variational} use this principle to propose Bayesian
Dropout inference methods such as MC Dropout and Variational Dropout.  Concrete dropout
\cite{gal2017concrete} extends this idea to optimize the dropout probabilities. Variational Structured Dropout \cite{nguyen2021structured} employs Householder transformation to learn a structured representation for multiplicative Gaussian noise in the Variational Dropout method.    

\subsection{Flat Minima} 
Flat minimizers have been found to improve the generalization ability of neural networks. This is because they enable models to find wider local minima, which makes them more robust against shifts between train and test sets~\cite{DBLP:conf/iclr/JiangNMKB20, DBLP:conf/nips/PetzkaKASB21, DBLP:conf/uai/DziugaiteR17, vaOT-MDR2023}. The relationship between generalization ability and the width of minima has been investigated theoretically and empirically in many studies, notably~\cite{DBLP:conf/nips/HochreiterS94, neyshabur2017exploring, dinh2017sharp, fort2019emergent}. Moreover, various methods seeking flat minima have been proposed in~\cite{DBLP:conf/iclr/PereyraTCKH17, Chaudhari2017EntropySGDBG, DBLP:conf/iclr/KeskarMNST17, DBLP:conf/uai/IzmailovPGVW18, foret2021sharpnessaware, vaOT-MDR2023}.
Typically,~\cite{DBLP:conf/iclr/KeskarMNST17,Jastrzebski2017ThreeFI, wei2020implicit} investigate the impacts of different training factors such as batch size, learning rate, covariance of gradient, and dropout on the flatness of found minima. Additionally, several approaches pursue wide local minima by adding regularization terms to the loss function~\cite{DBLP:conf/iclr/PereyraTCKH17, Zhang2018DeepML, zhang2019your, Chaudhari2017EntropySGDBG}. Examples of such regularization terms include softmax output's low entropy penalty~\cite{DBLP:conf/iclr/PereyraTCKH17} and distillation losses~\cite{Zhang2018DeepML, zhang2019your}.

    

    SAM, a method that aims to minimize the worst-case loss around the current model by seeking flat regions, has recently gained attention due to its scalability and effectiveness compared to previous methods~\cite{foret2021sharpnessaware, truong2023rsam}. SAM has been widely applied in various domains and tasks, such as meta-learning bi-level optimization~\cite{abbas2022sharp}, federated learning~\cite{qu2022generalized}, multi-task learning \cite{phan2022improving}, where it achieved tighter convergence rates and proposed generalization bounds. SAM has also demonstrated its generalization ability in vision models~\cite{chen2021vision}, language models~\cite{bahri-etal-2022-sharpness}, domain generalization~\cite{cha2021swad}, and multi-task learning~\cite{phan2022improving}. Some researchers have attempted to improve SAM by exploiting its geometry~\cite{kwon2021asam, kim22f}, additionally minimizing the surrogate gap~\cite{zhuang2022surrogate}, and speeding up its training time~\cite{du2022sharpness, liu2022towards}. Regarding the behavior of SAM, \cite{kaddour2022when} empirically studied the difference in sharpness obtained by SAM \cite{foret2021sharpnessaware} and SWA \cite{izmailov2018averaging}, \cite{mollenhoff2022sam} showed that SAM is an optimal Bayes relaxation of the standard Bayesian inference with a normal posterior, while \cite{vaOT-MDR2023} proved that distribution robustness \cite{blanchet2019quantifying,phan23a} is a probabilistic extension of SAM.

\section{Proposed Framework}
\label{sec:framework}
In what follows, we present the technicality of our proposed sharpness-aware posterior. Particularly, Section \ref{subsec:motivation} introduces the problem setting and motivation for our sharpness-aware posterior. Section \ref{subsec:theory_development} is dedicated to our theory development, while Section \ref{sec:sa_Bayesian_setting} is used to describe the Bayesian setting and variational inference approach for our sharpness-aware posterior. 
\subsection{Problem Setting and Motivation} \label{subsec:motivation}
We aim to develop Sharpness-Aware Bayesian Neural Networks
(SA-BNN). Consider a family of neural networks $f_{\theta}(x)$ with $\theta\in\Theta$
and a training set $\mathcal{S}=\{(x_{1},y_{1}),...,(x_{n},y_{n})\}$
where $\left(x_{i},y_{i}\right)\sim\mathcal{D}$. We wish to learn
a posterior distribution $\mathbb{Q}_{S}^{SA}$ with the density function
$q^{SA}(\theta|\mathcal{S})$ such that any model $\theta\sim\mathbb{Q}^{SA}_S$
is aware of the sharpness when predicting over the training set $\mathcal{S}$. 

We depart with the standard posterior 
\[
q(\theta\mid\mathcal{S})\propto\prod_{i=1}^{n}p(y_{i}\mid x_{i}, \mathcal{S},\theta)p(\theta),
\]
 where the prior distribution $\mathbb{P}$ has the density function
$p(\theta)$ and the likelihood has the form 
\begin{align*}
p\left(y\mid x,\mathcal{S},\theta\right) & \propto\exp\left\{ -\frac{\lambda}{|\mathcal{S}|}\ell\left(f_{\theta}(x),y\right)\right\}
  =\exp\left\{ -\frac{\lambda}{n}\ell\left(f_{\theta}(x),y\right)\right\} 
\end{align*}
with the loss function $\ell$. The standard posterior $\mathbb{Q}_{\mathcal{S}}$
has the density function defined as
\begin{equation}
q(\theta\mid\mathcal{S})\propto\exp\left\{ -\frac{\lambda}{n}\sum_{i=1}^{n}\ell\left(f_{\theta}\left(x_{i}\right),y_{i}\right)\right\} p(\theta),\label{eq:posterior}
\end{equation}
where $\lambda\geq0$ is a regularization parameter.

We define the general and empirical losses as follows:
\[
\mathcal{L}_{\mathcal{D}}\left(\theta\right)=\mathbb{E}_{\left(x,y\right)\sim\mathcal{D}}\left[\ell\left(f_{\theta}\left(x\right),y\right)\right].
\]
\[
\mathcal{L}_{\mathcal{S}}\left(\theta\right)=\mathbb{E}_{\left(x,y\right)\sim\mathcal{\mathcal{S}}}\left[\ell\left(f_{\theta}\left(x\right),y\right)\right]=\frac{1}{n}\sum_{i=1}^{n}\ell\left(f_{\theta}\left(x_{i}\right),y_{i}\right).
\]
Basically, the general loss is defined as the expected loss over the entire data-label distribution $\mathcal{D}$, while the empirical loss is defined as the empirical loss over a specific training set $\mathcal{S}$.

The standard posterior in Eq. (\ref{eq:posterior}) can be rewritten
as
\begin{equation}
q(\theta\mid\mathcal{S})\propto\exp\left\{ -\lambda\mathcal{L}_{\mathcal{S}}\left(\theta\right)\right\} p(\theta).\label{eq:posterior_loss}
\end{equation}

Given a distribution $\mathbb{Q}$ with the density function $q\left(\theta\right)$
over the model parameters $\theta\in\Theta$, we define the empirical
and general losses over this model distribution $\mathbb{Q}$ as
\[
\mathcal{L_{S}}\left(\mathbb{Q}\right)=\int_{\Theta}\mathcal{L}_{\mathcal{S}}\left(\theta\right)d\mathbb{Q}\left(\theta\right)=\int_{\Theta}\mathcal{L}_{\mathcal{S}}\left(\theta\right)q\left(\theta\right)d\theta.
\]
\[
\mathcal{L_{D}}\left(\mathbb{Q}\right)=\int_{\Theta}\mathcal{L}_{\mathcal{D}}\left(\theta\right)d\mathbb{Q}\left(\theta\right)=\int_{\Theta}\mathcal{L}_{\mathcal{D}}\left(\theta\right)q\left(\theta\right)d\theta.
\]
Specifically, the general loss over the model distribution $\mathbb{Q}$ is defined as the expectation of the general losses incurred by the models sampled from this distribution, while the empirical loss over the model distribution $\mathbb{Q}$ is defined as the expectation of the empirical losses incurred by the models sampled from this distribution.

\subsection{Our Theory Development} \label{subsec:theory_development}
We now present the theory development for the sharpness-aware posterior whose proofs can be found in the supplementary material. Inspired by the Gibbs form of the standard posterior $\mathbb{Q}_{\mathcal{S}}$
in Eq. (\ref{eq:posterior_loss}), we establish the following theorem
to connect the standard posterior $\mathbb{Q}_{\mathcal{S}}$ with
the density $q(\theta\mid\mathcal{S})$ and the empirical loss $\mathcal{L_{S}}\left(\mathbb{Q}\right)$ \cite{catoni2007pac, alquier2016properties}.

\begin{theorem}
\label{thm:Gibbs_posterior}Consider the following optimization problem
\begin{equation}
\min_{\mathbb{Q<<\mathbb{P}}}\left\{ \lambda\mathcal{L}_{S}\left(\mathbb{Q}\right)+KL\left(\mathbb{Q},\mathbb{P}\right)\right\} ,\label{op:empirical}
\end{equation}
where we search over $\mathbb{Q}$ absolutely continuous w.r.t. $\mathbb{P}$
and $KL\left(\cdot,\cdot\right)$ is the Kullback-Leibler divergence.
This optimization has a closed-form optimal solution $\mathbb{Q}^{*}$
with the density
\[
q^{*}\left(\theta\right)\propto\exp\left\{ -\lambda\mathcal{L}_{\mathcal{S}}\left(\theta\right)\right\} p(\theta),
\]
which is exactly the standard posterior $\mathbb{Q}_{\mathcal{S}}$
with the density $q(\theta\mid\mathcal{S})$.
\end{theorem}

Theorem \ref{thm:Gibbs_posterior} reveals that we need to find the
posterior $\mathbb{Q}_{\mathcal{S}}$ balancing between optimizing
its empirical loss $\mathcal{L}_{S}\left(\mathbb{Q}\right)$ and simplicity via
$KL\left(\mathbb{Q},\mathbb{P}\right)$. However, minimizing the empirical
loss $\mathcal{L}_{S}\left(\mathbb{Q}\right)$ only ensures the correct
predictions for the training examples in $\mathcal{S}$, hence possibly
encountering overfitting. Therefore, it is desirable to replace the empirical
loss by the general loss to combat overfitting. 

To mitigate overfitting, in (\ref{op:empirical}), we replace the
empirical loss by the general loss and solve the following
optimization problem (OP):
\begin{equation}
\min_{\mathbb{Q<<\mathbb{P}}}\left\{ \lambda\mathcal{L}_{\mathcal{D}}\left(\mathbb{Q}\right)+KL\left(\mathbb{Q},\mathbb{P}\right)\right\} .\label{op:general}
\end{equation}

Notably, solving the optimization problem (OP) in (\ref{op:general})
is generally intractable. To make it tractable, we find its upper-bound
which is relevant to the sharpness of a distribution $\mathbb{Q}$ over models as shown in the following theorem.
\begin{theorem}
\label{thm:general_loss}Assume that $\Theta$ is a compact set. Under some mild conditions, given
any $\delta\in[0;1]$, with the probability at least $1-\delta$ over
the choice of $\mathcal{S}\sim\mathcal{D}^{n}$, for any distribution
$\mathbb{Q}$, we have
\[
\mathcal{L}_{\mathcal{D}}\left(\mathbb{Q}\right)\leq\mathbb{E}_{\theta\sim\mathbb{Q}}\left[\max_{\theta':\norm{\theta'-\theta}\leq\rho}\mathcal{L}_{\mathcal{S}}\left(\theta'\right)\right]+f\left(\max_{\theta\in\Theta}\norm{\theta}^{2},n\right),
\]
where $f$ is a non-decreasing function w.r.t. the first variable and approaches $0$ when the training size $n$ approaches $\infty$.
\end{theorem}

We note that the proof of Theorem \ref{thm:general_loss} is not a trivial extension of sharpness-aware minimization because we need to tackle the general and empirical losses over a distribution $\mathbb{Q}$. To make explicit our sharpness over a distribution $\mathbb{Q}$ on models, we rewrite the upper-bound of the inequality as
\[
\mathbb{E}_{\theta\sim\mathbb{Q}}\left[\max_{\theta':\norm{\theta'-\theta}\leq\rho}\mathcal{L}_{\mathcal{S}}\left(\theta'\right)-\mathcal{L}_{\mathcal{S}}\left(\theta\right)\right]+\mathcal{L}_{\mathcal{S}}\left(\mathbb{Q}\right)+f\left(\max_{\theta\in\Theta}\norm{\theta}^{2},n\right),
\]
where  the first term $\mathbb{E}_{\theta\sim\mathbb{Q}}\left[\max_{\theta':\norm{\theta'-\theta}\leq\rho}\mathcal{L}_{\mathcal{S}}\left(\theta'\right)-\mathcal{L}_{\mathcal{S}}\left(\theta\right)\right]$ can be regarded as \textit{the sharpness over the distribution $\mathbb{Q}$ on the model space} and the last term $f\left(\max_{\theta\in\Theta}\norm{\theta}^{2},n\right)$ is a constant.   

Moreover, inspired by Theorem \ref{thm:general_loss}, we propose solving the
following OP which forms an upper-bound of the desirable OP in (\ref{op:general})
\begin{equation}
\min_{\mathbb{Q<<\mathbb{P}}}\left\{ \lambda\mathbb{E}_{\theta\sim\mathbb{Q}}\left[\max_{\theta':\norm{\theta'-\theta}\leq\rho}\mathcal{L}_{\mathcal{S}}\left(\theta'\right)\right]+KL\left(\mathbb{Q},\mathbb{P}\right)\right\} .\label{eq:Bayes_SAM_OP}
\end{equation}

The following theorem characterizes the optimal solution of the OP
in (\ref{eq:Bayes_SAM_OP}).
\begin{theorem}
    \label{thm:Bayses_SAM_solution}The optimal solution the OP in (\ref{eq:Bayes_SAM_OP})
is the sharpness-aware posterior distribution $\mathbb{Q}_{S}^{SA}$
with the density function $q^{SA}(\theta|\mathcal{S})$:
\begin{align*}
q^{SA}(\theta|\mathcal{S}) & \propto\exp\left\{ -\lambda\max_{\theta':\norm{\theta'-\theta}\leq\rho}\mathcal{L}_{\mathcal{S}}\left(\theta'\right)\right\} p\left(\theta\right)
 =\exp\left\{ -\lambda\mathcal{L}_{\mathcal{S}}\left(s\left(\theta\right)\right)\right\} p\left(\theta\right),
\end{align*}
where we have defined $s\left(\theta\right)=\argmax{\theta':\norm{\theta'-\theta}\leq\rho}\mathcal{L}_{\mathcal{S}}\left(\theta'\right)$.
\end{theorem}

Theorem \ref{thm:Bayses_SAM_solution} describes the close form of the sharpness-aware posterior distribution $\mathbb{Q}_{S}^{SA}$
with the density function $q^{SA}(\theta|\mathcal{S})$. Based on this characterization, in what follows, we introduce the SA Bayesian setting that sheds lights on its variational approach.

\subsection{Sharpness-Aware Bayesian Setting and Its Variational Approach} \label{sec:sa_Bayesian_setting}
\textbf{Bayesian Setting:} To promote the Bayesian setting for sharpness-aware posterior distribution
$\mathbb{Q}_{S}^{SA}$, we examine the sharpness-aware likelihood
\begin{align*}
p^{SA}\left(y\mid x,\mathcal{S},\theta\right) & \propto\exp\left\{ -\frac{\lambda}{|\mathcal{S}|}\ell\left(f_{s\left(\theta\right)}(x),y\right)\right\} 
  =\exp\left\{ -\frac{\lambda}{n}\ell\left(f_{s\left(\theta\right)}(x),y\right)\right\} ,
\end{align*}
where $s\left(\theta\right)=\argmax{\theta':\norm{\theta'-\theta}\leq\rho}\mathcal{L}_{\mathcal{S}}\left(\theta'\right)$.

With this predefined sharpness-aware likelihood, we can recover the
sharpness-aware posterior distribution $\mathbb{Q}_{S}^{SA}$ with
the density function $q^{SA}(\theta|\mathcal{S})$:
\[
q^{SA}(\theta|\mathcal{S})\propto\prod_{i=1}^{n}p^{SA}\left(y_{i}\mid x_{i},\mathcal{S},\theta\right)p\left(\theta\right).
\]
\textbf{Variational inference for the sharpness-aware posterior distribution:} We now develop the variational inference for the sharpness-aware posterior
distribution. Let denote $X=\left[x_{1},...,x_{n}\right]$ and $Y=\left[y_{1},...,y_{n}\right]$.
Considering an approximate posterior family $\left\{ q_{\phi}\left(\theta\right):\phi\in\Phi\right\} $,
we have
\begin{align*}
 & \log p^{SA}\left(Y\mid X,\mathcal{S}\right)=\int_{\Theta}q_{\phi}\left(\theta\right)\log p^{SA}\left(Y\mid X,\mathcal{S}\right)d\theta\\
 & =\int_{\Theta}q_{\phi}\left(\theta\right)\log\frac{p^{SA}\left(Y\mid\theta,X,\mathcal{S}\right)p\left(\theta\right)}{q_{\phi}\left(\theta\right)}\frac{q_{\phi}\left(\theta\right)}{q^{SA}(\theta|\mathcal{S})}d\theta\\
 & =\mathbb{E}_{q_{\phi}\left(\theta\right)}\left[\sum_{i=1}^{n}\log p^{SA}\left(y_{i}\mid x_{i},\mathcal{S},\theta\right)\right]-KL\left(q_{\phi},p\right) +KL\left(q_{\phi},q^{SA}\right).
\end{align*}

It is obvious that we need to maximize the following lower bound for
maximally reducing the gap $KL\left(q_{\phi},q^{SA}\right)$:
\[
\max_{q_{\phi}}\left\{ \mathbb{E}_{q_{\phi}\left(\theta\right)}\left[\sum_{i=1}^{n}\log p^{SA}\left(y_{i}\mid x_{i},\mathcal{S},\theta\right)\right]-KL\left(q_{\phi},p\right)\right\},
\]
which can be equivalently rewritten as
\begin{align}
 & \min_{q_{\phi}}\left\{ \lambda\mathbb{E}_{q_{\phi}\left(\theta\right)}\left[\mathcal{L}_{\mathcal{S}}\left(s\left(\theta\right)\right)\right]+KL\left(q_{\phi},p\right)\right\} \text{or} \nonumber \\
 & \min_{q_{\phi}}\left\{ \lambda\mathbb{E}_{q_{\phi}\left(\theta\right)}\left[\max_{\theta':\norm{\theta'-\theta}\leq\rho}\mathcal{L}_{\mathcal{S}}\left(\theta'\right)\right]+KL\left(q_{\phi},p\right)\right\}. \label{eq:variational}
\end{align}

\textbf{Derivation for Variational Approach with A Gaussian Approximate Posterior:} Inspired by the geometry-based SAM approaches ~\cite{kwon2021asam, kim22f}, we incorporate
the geometry to the SA variational approach via the distance to define
the ball for the sharpness as $\norm{\theta'-\theta}_{\text{diag}(T_{\theta})}=\sqrt{\left(\theta'-\theta\right)^{T}\text{diag}(T_{\theta})^{-1}\left(\theta'-\theta\right)}$
as
\[
\min_{q_{\phi}}\left\{ \lambda\mathbb{E}_{q_{\phi}\left(\theta\right)}\left[\max_{\theta':\norm{\theta'-\theta}_{\text{diag}(T_{\theta})}\leq\rho}\mathcal{L}_{\mathcal{S}}\left(\theta'\right)\right]+KL\left(q_{\phi},p\right)\right\} .
\]

To further clarify, we consider our SA posterior distribution to Bayesian
NNs, wherein we impose the Gaussian distributions to its weight matrices
$W_{i}\sim\mathcal{N}\left(\mu_{i},\sigma_{i}^{2}\mathbb{I}\right),i=1,\dots,L$\footnote{We absorb the biases to the weight matrices.}.
The parameter $\phi$ consists of $\mu_{i},\sigma_{i},i=1,\dots,L$.
For $\theta=W_{1:L}\sim q_{\phi}$, using the reparameterization trick
$W_{i}=\mu_{i}+\text{diag}(\sigma_{i})\epsilon_{i},\epsilon_{i}\sim\mathcal{N}\left(0,\mathbb{I}\right)$
and by searching $\theta^{'}=W_{1:L}^{'}$ with $W_{i}^{'}=\mu_{i}^{'}+\text{diag}(\sigma_{i})\epsilon_{i},\epsilon_{i}\sim\mathcal{N}\left(0,\mathbb{I}\right)$,
the constraint $\norm{\theta-\theta'}_{\text{diag}(T_{\theta})}=\norm{\mu-\mu'}_{\text{diag}(T_{\theta})}$
with $\mu=\mu_{1:L}$ and $\mu^{'}=\mu_{1:L}^{'}$. Thus, the
OP in (\ref{eq:variational}) reads
\begin{equation}
\min_{\mu,\sigma}\left\{ \lambda\mathbb{E}_{\epsilon}\left[\max_{\norm{\mu^{'}-\mu}_{\text{diag}(T_{\mu,\sigma})}\leq\rho}\mathcal{L}_{\mathcal{S}}\left(\left[\mu_{i}^{'}+\text{diag}(\sigma_{i})\epsilon_{i}\right]_{i=1}^{L}\right)\right]\right\} ,\label{eq:variational_derive}
\end{equation}
where $\sigma=\sigma_{1:L}$, $\epsilon=\epsilon_{1:L}$, and we define
$\text{diag}(T_{\theta})=\text{diag}(T_{\mu,\sigma})$ in the distance of the geometry.

To solve the OP in (\ref{eq:variational_derive}), we sample $\epsilon=\epsilon_{1:L}$
from the standard Gaussian distributions, employ an one-step gradient
ascent to find $\mu^{'}$, and use the gradient at $\mu^{'}$ to update
$\mu$. Specifically, we find $\mu'$ \cite{convex_op} (Chapter 9) as
\[
\mu'= \mu + \rho\frac{\text{diag}(T_{\mu,\sigma})\nabla_{\mu}\mathcal{L}_{\mathcal{S}}\left(\left[\mu_{i}+\text{diag}(\sigma_{i})\epsilon_{i}\right]_{i=1}^{L}\right)}{\norm{\text{diag}(T_{\mu,\sigma})\nabla_{\mu}\mathcal{L}_{\mathcal{S}}\left(\left[\mu_{i}+\text{diag}(\sigma_{i})\epsilon_{i}\right]_{i=1}^{L}\right)}}.
\]

The diagnose of $\text{diag}(T_{\mu,\sigma})$ specifies the importance
level of the model weights, i.e., the weight with a higher importance
level is encouraged to have a higher sharpness via a smaller absolute
partial derivative of the loss w.r.t. this weight. We consider $\text{diag}(T_{\mu,\sigma})=\mathbb{I}$
(i.e., \textit{the standard SA BNN}) and $\text{diag}(T_{\mu,\sigma})=\text{diag}\left(\frac{\left|\mu\right|}{\sigma}\right)$
(i.e., \textit{the geometry SA BNN}). Here we note that $\frac{\bigcdot}{\bigcdot}$
represents the element-wise division.

Finally, the objective function in (\ref{eq:variational}) indicates that we aim to find an approximate posterior distribution that ensures any model sampled from it is aware of the sharpness, while also preferring simpler approximate posterior distributions. This preference can be estimated based on how we equip these distributions. With the Bayesian setting and variational inference formulation, our proposed sharpness-aware posterior can be integrated into MCMC-based and variational inference-based Bayesian Neural Networks. The supplementary material contains the details on how to derive variational approaches and incorporate the sharpness-awareness into the BNNs used in our experiments including BNNs with an approximate Gaussian distribution \cite{kingma2015variational}, BNNs with stochastic gradient Langevin dynamics
(SGLD) \cite{welling2011bayesian}, MC-Dropout \cite{gal2016-mcdropout}, Bayesian deep ensemble \cite{lakshminarayanan2017simple}, and SWAG \cite{Maddox2019-swag}.





\section{Experiments}
\label{sec:exp}
In this section, we conduct various experiments to demonstrate the effectiveness of the sharpness-aware approach on Bayesian Neural networks, including BNNs with an approximate Gaussian distribution \cite{kingma2015variational} (i.e., SGVB for model's reparameterization trick and SGVB-LRT for representation's reparameterization trick), BNNs with stochastic gradient Langevin dynamics
(SGLD) \cite{welling2011bayesian}, MC-Dropout \cite{gal2016-mcdropout}, Bayesian deep ensemble \cite{lakshminarayanan2017simple}, and SWAG \cite{Maddox2019-swag}.  The experiments are conducted on three benchmark datasets: CIFAR-10, CIFAR-100, and ImageNet ILSVRC-2012, and report accuracy, negative log-likelihood (NLL), and Expected Calibration Error (ECE) to estimate the calibration capability and uncertainty of our method against baselines. The details of the dataset and implementation are described in the supplementary material\footnote{The implementation is provided in \url{https://github.com/anh-ntv/flat_bnn.git}}.

\begin{table}[t!]
\centering
\caption{Classification score on CIFAR-100 dataset.Each experiment is repeated three times with different random seeds and reports the mean and standard deviation.}

\label{tab:ex-CIFAR-100}
\resizebox{1.\columnwidth}{!}{\begin{tabular}{lccc|ccc}
\midrule
  & \multicolumn{3}{c}{\textbf{PreResNet-164}}      & \multicolumn{3}{c}{\textbf{WideResNet28x10}} \\
Method  & ACC $\uparrow$   & NLL $\downarrow$  & ECE $\downarrow$  & ACC $\uparrow$   & NLL $\downarrow$  & ECE $\downarrow$  \\ \midrule \midrule

\multicolumn{2}{l}{\textbf{Variational inference}}  &  &  &  &  &  \\
MC-Dropout    &  79.50 $\pm$ 0.37  &  0.9162 $\pm$ 0.0103   & 0.0993 $\pm$ 0.0033 &  82.30 $\pm$ 0.19  &  0.6500 $\pm$ 0.0049 & 0.0574 $\pm$  0.0028 \\
F-MC-Dropout     &  \textbf{81.06 $\pm$ 0.44}  &   \textbf{0.7027 $\pm$ 0.0049}  &  \textbf{0.0514 $\pm$ 0.0047}  &  \textbf{83.24 $\pm$ 0.11}  & \textbf{0.6144 $\pm$ 0.0068 } &  \textbf{0.0250 $\pm$ 0.0027} \\ \cmidrule{1-7}
Deep-ens    & 82.08 $\pm$ 0.42 & 0.7189  $\pm$ 0.0108 & 0.0334  $\pm$ 0.0064  &  83.04 $\pm$ 0.15 & 0.6958 $\pm$ 0.0335 & 0.0483 $\pm$ 0.0017    \\
F-Deep-ens      & \textbf{82.54 $\pm$ 0.10 } & \textbf{0.6286 $\pm$ 0.0022} & \textbf{0.0143 $\pm$ 0.0041}  &  \textbf{84.52 $\pm$ 0.03} & \textbf{0.5644 $\pm$ 0.0106} & \textbf{0.0191 $\pm$ 0.0039}  \\ \midrule \midrule

\multicolumn{2}{l}{\textbf{Markov chain Monte Carlo}}  &  &  &  &  & \\
SGLD   &  80.13 $\pm$ 0.01  &  0.7604 $\pm$  0.0010   &  0.1161 $\pm$ 0.0031  & 81.38 $\pm$ 0.10 &  0.7123 $\pm$ 0.0204 &   0.0958 $\pm$ 0.0004  \\
F-SGLD     &  \textbf{80.82 $\pm$ 0.02}  & \textbf{ 0.7276 $\pm$  0.0012}  & \textbf{ 0.1085 $\pm$ 0.0008}  &  \textbf{82.12 $\pm$ 0.16}  & \textbf{0.6722 $\pm$ 0.0112}  &   \textbf{0.0820 $\pm$ 0.0021}   \\ \midrule \midrule

\multicolumn{2}{l}{\textbf{Sample}}  &  &  &  &  & \\
SWAG-Diag    &  80.18 $\pm$ 0.50  &  0.6837 $\pm$   0.0186   & \textbf{ 0.0239 $\pm$ 0.0047 } &  82.40 $\pm$ 0.09  & 0.6150 $\pm$  0.0029  & 0.0322 $\pm$ 0.0018 \\
F-SWAG-Diag   & \textbf{  81.01 $\pm$	0.29 } & \textbf{  0.6645 $\pm$	0.0050  } & 0.0242 $\pm$	0.0039 & \textbf{   83.50 $\pm$ 0.29} & \textbf{ 0.5763 $\pm$	0.0120 } & \textbf{   0.0151 $\pm$	0.0020}  \\ \midrule
SWAG  &  79.90 $\pm$ 0.50  &  \textbf{0.6595 $\pm$ 0.0019 }  &  0.0587 $\pm$ 0.0048  &  82.23 $\pm$ 0.19  & 0.6078 $\pm$ 0.0006  &  \textbf{0.0113 $\pm$ 0.0020} \\
F-SWAG   & \textbf{  80.93 $\pm$ 0.27} &  0.6704 $\pm$	0.0049 & \textbf{  0.0350 $\pm$	0.0025 } & \textbf{  83.57 $\pm$ 0.26 } & \textbf{ 0.5757 $\pm$ 0.0136 } & 0.0196 $\pm$ 0.0015 \\ \cline{1-7}

\bottomrule
\end{tabular}}

\end{table}

\begin{table*}
\centering
\caption{Classification score on CIFAR-10 dataset.Each experiment is repeated three times with different random seeds and reports the mean and standard deviation.}

\label{tab:ex-CIFAR-10}
\resizebox{1\columnwidth}{!}{\begin{tabular}{lccc|ccc}
\midrule
  & \multicolumn{3}{c}{\textbf{PreResNet-164}}      & \multicolumn{3}{c}{\textbf{WideResNet28x10}} \\
Method  & ACC $\uparrow$   & NLL $\downarrow$  & ECE $\downarrow$  & ACC $\uparrow$   & NLL $\downarrow$  & ECE $\downarrow$  \\ \midrule \midrule

\multicolumn{2}{l}{\textbf{Variational inference}}   &  &  &  &  &  \\
MC-Dropout    & 96.18 $\pm$ 0.02 & 0.1270 $\pm$ 0.0030 &  0.0162 $\pm$ 0.0007       & 96.39 $\pm$ 0.09 & 0.1094 $\pm$ 0.0021 &  \textbf{0.0094 $\pm$ 0.0014}  \\
F-MC-Dropout     & \textbf{96.39 $\pm$ 0.18} & \textbf{0.1137 $\pm$ 0.0024}          & \textbf{0.0118 $\pm$ 0.0006}       &\textbf{ 97.10 $\pm$ 0.12} & \textbf{0.0966 $\pm$ 0.0047} &  0.0095 $\pm$ 0.0008  \\ \midrule
Deep-ens    & 96.39 $\pm$ 0.09 & 0.1277 $\pm$ 0.0030 & 0.0108 $\pm$ 0.0015 &  96.96 $\pm$ 0.10 & 0.1031 $\pm$ 0.0076 & 0.0087 $\pm$ 0.0018  \\
F-Deep-ens      & \textbf{96.70  $\pm$ 0.04} & \textbf{0.1031 $\pm$ 0.0016} & \textbf{0.0057 $\pm$ 0.0031}  &  \textbf{97.11 $\pm$ 0.10} & \textbf{0.0851 $\pm$ 0.0011} & \textbf{0.0059 $\pm$ 0.0012} \\ \midrule \midrule

\multicolumn{2}{l}{\textbf{Markov chain Monte Carlo}}  &  &  &  &  &  \\
SGLD   & 94.79 $\pm$  0.10 &  0.2089 $\pm$ 0.0021 &  0.0711 $\pm$ 0.0061 &  95.87 $\pm$  0.08 & 0.1573 $\pm$  0.0190 &   0.0463 $\pm$  0.0050 \\
F-SGLD     & \textbf{95.04 $\pm$ 0.06} &  \textbf{0.1912 $\pm$ 0.0080} &  \textbf{0.0601 $\pm$ 0.0002}  &  \textbf{96.43 $\pm$ 0.05}  & \textbf{0.1336 $\pm$ 0.004}  &    \textbf{0.0385 $\pm$ 0.0003}  \\ \midrule \midrule

\multicolumn{2}{l}{\textbf{Sample}}   &  &  &  &  &  \\
SWAG-Diag    & 96.03 $\pm$ 0.10 & 0.1251 $\pm$ 0.0029          &  0.0082 $\pm$ 0.0008       & 96.41 $\pm$ 0.05 & 0.1077 $\pm$ 0.0009 &  0.0047 $\pm$ 0.0013  \\
F-SWAG-Diag   & \textbf{ 96.23 $\pm$ 0.01}& \textbf{ 0.1108 $\pm$	0.0013 }& \textbf{  0.0043 $\pm$ 	0.0005 }& \textbf{ 97.05 $\pm$	0.08 }& \textbf{ 0.0888 $\pm$ 0.0052   }& \textbf{  0.0043 $\pm$ 0.0004}     \\ \midrule
 SWAG   & 96.03 $\pm$ 0.02 & 0.1232 $\pm$ 0.0022          &  \textbf{0.0053 $\pm$ 0.0004 }      & 96.32 $\pm$ 0.08 & 0.1122 $\pm$ 0.0009 & 0.0088 $\pm$ 0.0006   \\
F-SWAG     & \textbf{ 96.25 $\pm$	0.03 } & \textbf{ 0.11062 $\pm$ 0.0014  } & 0.0056 $\pm$ 0.0002  & \textbf{ 97.09 $\pm$ 0.14 } & \textbf{ 0.0883 $\pm$ 0.0004 } & \textbf{  0.0036 $\pm$ 0.0008}  \\ 
\bottomrule
\end{tabular}}

\end{table*}

\begin{table*}[t!]
\centering
\caption{Classification scores of approximate the Gaussian posterior on the CIFAR datasets. Each experiment is repeated three times with different random seeds and reports the mean and standard deviation.}
\label{tab:ex-sgvb}

\resizebox{1.\columnwidth}{!}{
\begin{tabular}{lccc|ccc}
\midrule

 & \multicolumn{3}{c}{\textbf{Resnet10}}      & \multicolumn{3}{c}{\textbf{Resnet18}} \\
Method  & ACC $\uparrow$   & NLL $\downarrow$  & ECE $\downarrow$  & ACC $\uparrow$   & NLL $\downarrow$  & ECE $\downarrow$  \\ \midrule \midrule
\multicolumn{3}{l}{\textbf{Experiments on Cifar-100 dataset}} &  & &  \\
SGVB-LRT  &  61.75 $\pm$ 0.75 &  1.534 $\pm$ 0.03  &  0.0676 $\pm$ 0.01 &  68.95 $\pm$ 1.20 &  1.140 $\pm$ 0.21 &   0.063  $\pm$ 0.04  \\
F-SGVB-LRT   &  62.25 $\pm$ 0.57 &  1.4001 $\pm$ 0.04 &  0.0642 $\pm$ 0.01 &  70.00 $\pm$ 1.42  &  1.127  $\pm$ 0.25 &   \textbf{0.022 $\pm$ 0.05}       \\
\multicolumn{1}{r}{+ Geometry}   & \textbf{62.54  $\pm$ 0.67} & \textbf{1.3704  $\pm$  0.01} &	\textbf{0.0301 $\pm$ 0.03} &   \textbf{70.12  $\pm$ 1.02} &	\textbf{1.121 $\pm$ 0.23} &	0.036 $\pm$ 0.06  \\ \midrule
SGVB   & 54.40 $\pm$ 0.98 &  1.968 $\pm$ 0.05  & 0.214 $\pm$  0.00 &   60.91 $\pm$ 2.31 &  1.746 $\pm$ 0.15 &  0.246 $\pm$ 0.03    \\
F-SGVB       & 54.53 $\pm$ 0.33 & 1.967 $\pm$ 0.00 &	0.212 $\pm$ 0.00  &  61.54 $\pm$ 2.23 &	1.695 $\pm$ 0.15 & 0.242 $\pm$ 0.03 \\
\multicolumn{1}{r}{+ Geometry}   &\textbf{ 55.53 $\pm$ 0.65} &   \textbf{1.906 $\pm$ 0.02}   & \textbf{0.207 $\pm$  0.00} & \textbf{62.58 $\pm$ 0.53 }   & \textbf{1.612 $\pm$ 0.03} &    \textbf{0.224 $\pm$ 0.00}      \\  \midrule \midrule
\multicolumn{3}{l}{\textbf{Experiments on Cifar-10 dataset}}  &  & &  \\
SGVB-LRT   &  84.98 $\pm$ 1.87  &  0.422 $\pm$ 0.10   &  0.043 $\pm$ 0.04 &  89.10 $\pm$ 1.32  &  0.344 $\pm$ 0.02 &   0.033 $\pm$ 0.02   \\
F-SGVB-LRT     &  86.32 $\pm$ 1.34 &  0.409  $\pm$ 0.03 &  \textbf{0.017 $\pm$ 0.06} &  90.00 $\pm$ 1.10  &  0.291 $\pm$ 0.02 &   0.019 $\pm$ 0.01   \\
\multicolumn{1}{r}{+ Geometry}   & \textbf{86.44 $\pm$ 1.12} &	\textbf{0.403 $\pm$ 0.06} & 0.025 $\pm$  0.03 &   \textbf{90.31 $\pm$ 1.11} &	\textbf{0.262 $\pm$ 0.01} &	\textbf{0.014 $\pm$ 0.02}    \\  \midrule
SGVB   & 80.52  $\pm$  2.10 &  0.781  $\pm$ 0.23   & 0.237  $\pm$ 0.06 &   86.74 $\pm$ 1.25    & 0.541 $\pm$ 0.01 &  0.181 $\pm$ 0.02    \\
F-SGVB   &  80.60  $\pm$  1.88  &	0.776  $\pm$  0.13	& 0.223  $\pm$  0.05  &  \textbf{87.01 $\pm$ 0.91} &	0.534 $\pm$ 0.01	& 0.183  $\pm$  0.01 \\
\multicolumn{1}{r}{+ Geometry}  &  \textbf{82.05  $\pm$  0.47}    &	\textbf{0.704  $\pm$  0.01} & \textbf{0.206  $\pm$ 0.00}   & 86.80 $\pm$ 1.30 & \textbf{0.531 $\pm$ 0.01} & \textbf{0.175 $\pm$ 0.01}  \\
\bottomrule
\end{tabular}
}

\end{table*}

\begin{table}[t!]
\centering
\caption{Classification score on ImageNet dataset}
\label{tab:ex-imgagenet}
\resizebox{.65\columnwidth}{!}{
\begin{tabular}{lccc|ccc}
\midrule

  & \multicolumn{3}{c}{\textbf{Densenet-161}}      & \multicolumn{3}{c}{\textbf{ResNet-152}} \\
Model     & ACC $\uparrow$   & NLL $\downarrow$  & ECE $\downarrow$  & ACC $\uparrow$   & NLL $\downarrow$      & ECE $\downarrow$  \\ \hline \midrule
 SWAG-Diag   & 78.59  &     0.8559 & 0.0459    & 78.96  & 0.8584 & 0.0566  \\
F-SWAG-Diag   & \textbf{ 78.71 } & \textbf{0.8267 } & \textbf{0.0194} & \textbf{ 79.20 } & \textbf{ 0.8065} & \textbf{ 0.0199}   \\  \cmidrule{1-7}
 SWAG   & 78.59  & 0.8303& 0.0204&  79.08  & 0.8205&  0.0279\\
F-SWAG    & \textbf{ 78.70 } & \textbf{ 0.8262} & \textbf{ 0.0185} & \textbf{ 79.17 } & \textbf{ 0.8078} & \textbf{ 0.0208} \\ \midrule
 SGLD   & 78.50  & 0.8317 & \textbf{0.0157} &  79.00  & 0.8165 & 0.0220 \\
F-SGLD    & \textbf{78.64} & \textbf{ 0.8236} & 0.0166 & \textbf{ 79.16 } & \textbf{ 0.8050} & \textbf{ 0.0167} \\ \bottomrule
\end{tabular}}

\end{table}

\begin{figure}[b!]
    
     \centering
     \begin{subfigure}
         \centering
         \includegraphics[width=0.24\textwidth]{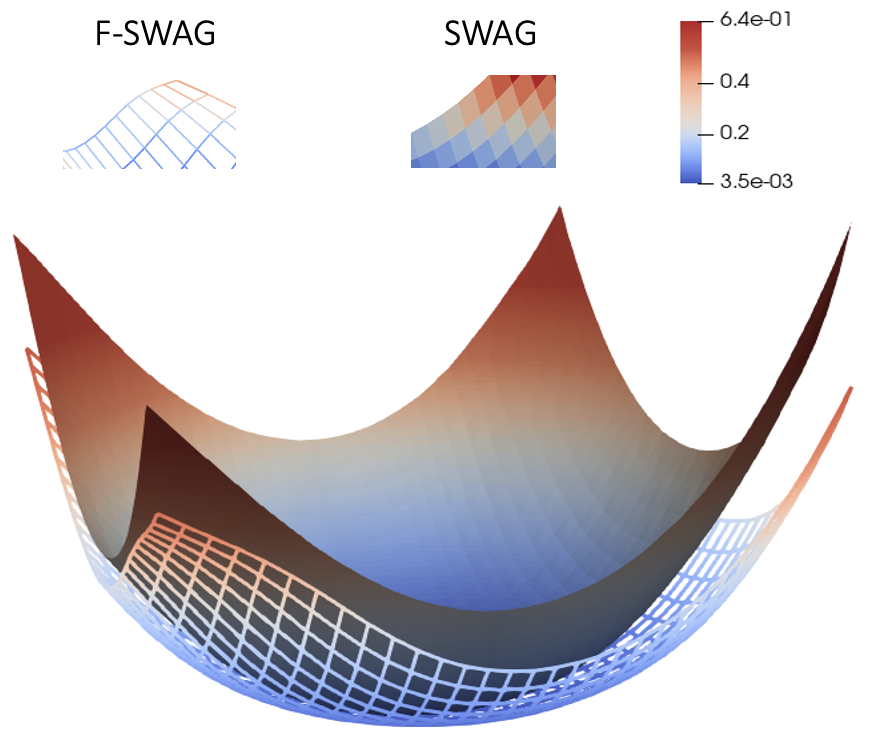}
     \end{subfigure}
     \hfill
     \begin{subfigure} 
         \centering
         \includegraphics[width=0.24\textwidth]{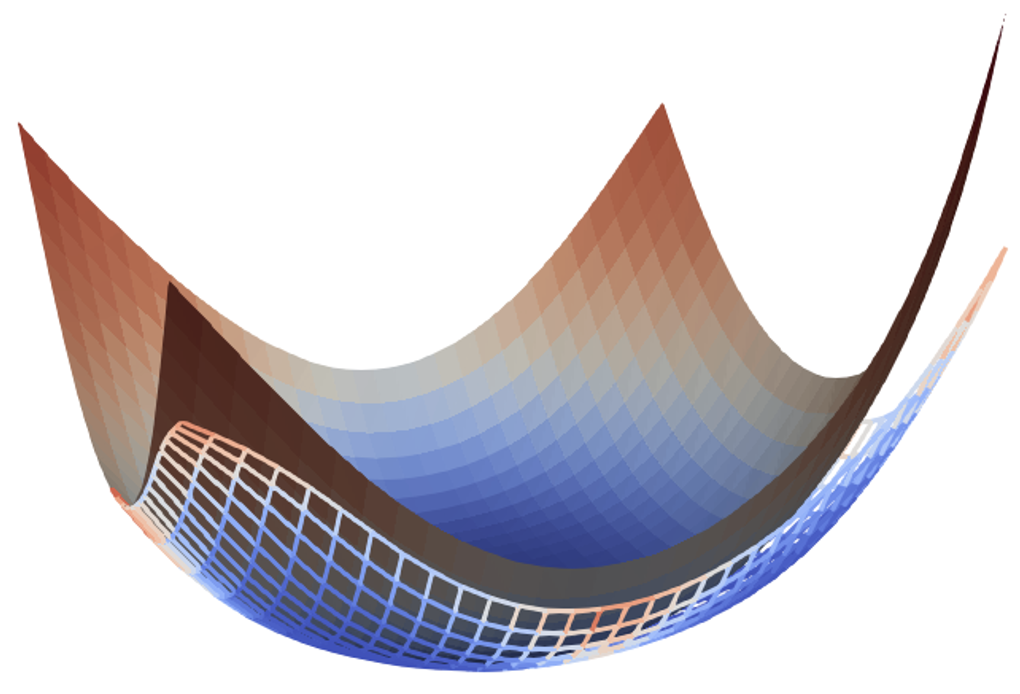}
     \end{subfigure}
     \hfill
     \begin{subfigure} 
         \centering
         \includegraphics[width=0.24\textwidth]{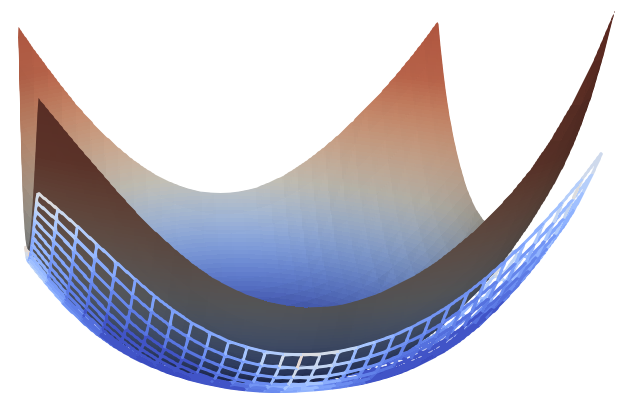}
     \end{subfigure}
     \hfill
     \begin{subfigure} 
         \centering
         \includegraphics[width=0.24\textwidth]{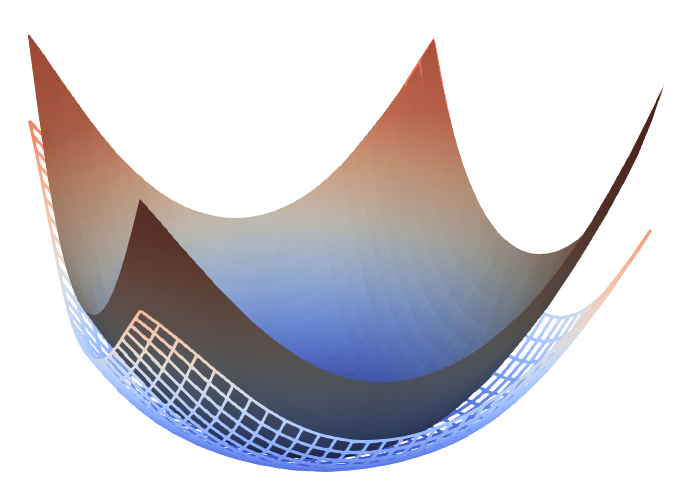}
     \end{subfigure}
        \caption{Comparing loss landscape of PreResNet-164 on CIFAR-100 dataset training with SWAG and F-SWAG method. For visualization purposes, we sample two models for each SWAG and F-SWAG and then plot the loss landscapes. It can be observed that the loss landscapes of our F-SWAG are flatter, supporting our argument for the flatter sampled models.}
        \label{fig:loss-landscape}
     
\end{figure}
     
     
     

\subsection{Experimental results} \label{subsec: exp-result}
\subsubsection{Predictive performance}


Our experimental results, presented in Tables \ref{tab:ex-CIFAR-100}, \ref{tab:ex-CIFAR-10}, \ref{tab:ex-sgvb} for CIFAR-100 and CIFAR-10 dataset, and Table \ref{tab:ex-imgagenet} for the ImageNet dataset, indicate a notable improvement across all experiments. 
It is worth noting that there is a trade-off between accuracy, negative log-likelihood, and expected calibration error. Nonetheless, our approach obtains a fine balance between these factors compared to the overall improvement.

\begin{figure}[h!]
     \centering
     \begin{subfigure}
         \centering
         \includegraphics[width=0.32\textwidth]{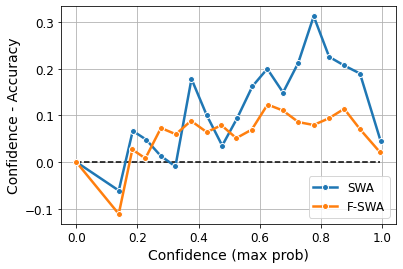}
     \end{subfigure}
     \begin{subfigure}
         \centering
         \includegraphics[width=0.32\textwidth]{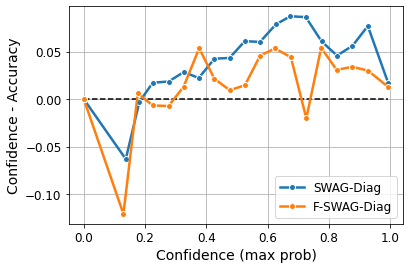}
     \end{subfigure}
     \begin{subfigure}
         \centering
         \includegraphics[width=0.32\textwidth]{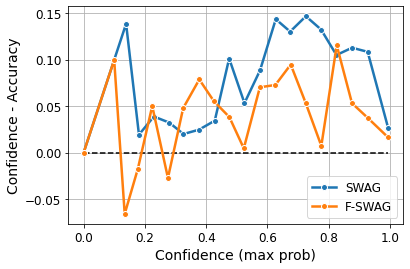}
     \end{subfigure}
     \centering
        \caption{Reliability diagrams for PreResNet164 on CIFAR-100. The confidence is split into 20 bins and plots the gap between confidence and accuracy in each bin. The best case is the black dashed line when this gap is zeros. The plots of F-SWAG get closer to the zero lines, implying our F-SWAG can calibrate the uncertainty better.}
         \label{fig:ece_bin}
\end{figure}

\subsection{Effectiveness of sharpness-aware posterior}
\label{subsec:exp-sharpness}

\textbf{Calibration of uncertainty estimates:}
We evaluate the ECE of each setting and compare it to baselines in Tables \ref{tab:ex-CIFAR-100}, \ref{tab:ex-CIFAR-10}, and \ref{tab:ex-imgagenet}. This score measures the maximum discrepancy between the accuracy and confidence of the model. To further clarify it, we display the Reliability Diagrams of PreResNet-164 on CIFAR-100 to understand how well the model predicts according to the confidence threshold in Figure \ref{fig:ece_bin}. The experiments is detailed in the supplementary material.

\textbf{Out-of-distribution prediction:}
The effectiveness of the sharpness-aware Bayesian neural network (BNN) is demonstrated in the above experiments, particularly in comparison to non-flat methods. In this section, we extend the evaluation to an out-of-distribution setting. Specifically, we utilize the BNN models trained on the CIFAR-10 dataset to assess their performance on the CIFAR-10-C dataset.
This is an extension of the CIFAR-10 designed to evaluate the robustness of machine learning models against common corruptions and perturbations in the input data. The corruptions include various forms of noise, blur, weather conditions, and digital distortions. We conduct an ensemble of 30 models sampled from the flat-posterior distribution and compared them with non-flat ones. We present the average result of each corruption group and the average result on the whole dataset in Table \ref{tab:ex-acc-CIFAR-10c}, the detailed result of each corruption form is displayed in the supplementary material. Remarkably, the flat BNN models consistently surpass their non-flat counterparts with respect to average ECE and accuracy metrics. This finding is additional evidence of the generalization ability of the sharpness-aware posterior.

\begin{table}[t!]
\centering
\caption{Classification score on CIFAR-10-C on PreResNet-164 model when training with CIFAR-10. The full result on each type of corruption is displayed in the supplemetary material.}
\label{tab:ex-acc-CIFAR-10c}
\resizebox{1.\columnwidth}{!}{%
\begin{tabular}{lcc|cc||cc|cc}
\midrule
& \multicolumn{4}{c}{\textbf{ECE $\downarrow$}} &  \multicolumn{4}{c}{\textbf{Accuracy $\uparrow$}} \\
Corruption &  SWAG-D & F-SWAG-D & SWAG & F-SWAG & 		SWAG-D & F-SWAG-D & SWAG & F-SWAG \\ \hline \midrule
Noise &  0.0729 & 0.0701 & 0.0958 & 0.0078 & 74.26 & 75.59 & 74.02 & 75.08 \\
Blur &  0.0121 & 0.0090 & 0.0202 & 0.0273 & 91.13 & 90.55 & 91.03 & 90.93 \\
Weather &  0.018 & 0.0142 & 0.0272 & 0.0240 & 89.47 & 89.18 & 89.42 & 89.11 \\
Digital and others &  0.0277 & 0.0229 & 0.0384 & 0.0209 & 87.03 & 86.94 & 86.93 & 87.19 \\
\midrule
Average &  0.0328 & \textbf{0.0290} & 0.0454 & \textbf{0.0200} & 85.47 & \textbf{85.56} & 85.35 & \textbf{85.58} \\ \bottomrule
\end{tabular}
}
\end{table}

\subsection{Ablation studies}

\label{subsec:losslandscape}

In Figure \ref{fig:loss-landscape}, we plot the loss-landscape of the models sampled from our proposal of sharpness-aware posterior against the non-sharpness-aware one. Particularly, we compare two methods F-SWAG and SWAG by selecting four random models sampled from the posterior distribution of each method under the same hyper-parameter settings. As observed, our method not only improves the generalization of ensemble inference, demonstrated by classification results in Section \ref{subsec: exp-result} and sharpness in Section \ref{subsec:exp-sharpness}, but also the individual sampled model is flatter itself.



We measure and visualize the sharpness of the models. To this end, we sample five models from the approximate posteriors and then take the average of the sharpness of these models. 
For a model $\theta$, the sharpness is evaluated as $\displaystyle\max_{||\epsilon||_2 \leq \rho}{\mathcal{L}_{\mathcal{S}}(\theta + \epsilon) - \mathcal{L}_{\mathcal{S}}(\theta)} $ to measure the change of loss value around $\theta$. We calculate the sharpness score of PreResNet-164 network for SWAG, and F-SWAG training on CIFAR-100 dataset and visualize them in the supplementary material. As shown there, the sharpness-aware versions produce smaller \textit{sharpness} scores compared to the corresponding baselines, indicating that our models get into flatter regions.


\section{Conclusion}
\label{sec:conslusion}
In this paper, we introduce theories in the Bayesian setting and discuss variational inference for the sharpness-aware posterior in the context of Bayesian Neural Networks (BNNs). The sharpness-aware posterior results in models that are less sensitive to noise and have a better generalization ability, as it enables the models sampled from it and the optimal approximate posterior estimates to have a higher flatness. We conducted extensive experiments that leveraged the sharpness-aware posterior with state-of-the-art Bayesian Neural Networks. Our main results show that the models sampled from the proposed posterior outperform their baselines in terms of ensemble accuracy, expected calibration error (ECE), and negative log-likelihood (NLL). This indicates that the flat-seeking counterparts are better at capturing the true distribution of weights in neural networks and providing accurate probabilistic predictions. Furthermore, we performed ablation studies to showcase the effectiveness of the flat posterior distribution on various factors such as uncertainty estimation, loss landscape, and out-of-distribution prediction. Overall, the sharpness-aware posterior presents a promising approach for improving the generalization performance of Bayesian neural networks.

\paragraph{Acknowledgements.} This work was partly supported by ARC DP23 grant DP230101176 and by the Air Force Office of Scientific Research under award number FA2386-23-1-4044.


\clearpage
\medskip
\bibliography{reference.bib}

\newpage
\appendix
\section{Theoretical Development}
\subsection{All Proofs}
\label{proofs}
\begin{theorem}
\label{thm:app_Gibbs_posterior_1}(Theorem 3.1 in the main paper) Consider the following optimization problem
\begin{equation}
\min_{\mathbb{Q<<\mathbb{P}}}\left\{ \lambda\mathcal{L}_{S}\left(\mathbb{Q}\right)+KL\left(\mathbb{Q},\mathbb{P}\right)\right\} ,\label{op:empirical}
\end{equation}
where we search over $\mathbb{Q}$ absolutely continuous w.r.t. $\mathbb{P}$
and $KL\left(\cdot,\cdot\right)$ is the Kullback-Leibler divergence.
This optimization has a closed-form optimal solution $\mathbb{Q}^{*}$
with the density
\[
q^{*}\left(\theta\right)\propto\exp\left\{ -\lambda\mathcal{L}_{\mathcal{S}}\left(\theta\right)\right\} p(\theta),
\]
which is exactly the standard posterior $\mathbb{Q}_{\mathcal{S}}$
with the density $q(\theta\mid\mathcal{S})$.
\end{theorem}
\begin{proof}
We have
\[
\lambda\mathcal{L}_{S}\left(\mathbb{Q}\right)+KL\left(\mathbb{Q},\mathbb{P}\right)=\lambda\int\mathcal{L}_{S}\left(\theta\right)q\left(\theta\right)d\theta+\int q(\theta)\log\frac{q\left(\theta\right)}{p\left(\theta\right)}d\theta.
\]

The Lagrange function is as follows
\[
L\left(q,\alpha\right)=\lambda\int\mathcal{L}_{S}\left(\theta\right)q\left(\theta\right)d\theta+\int q(\theta)\log\frac{q\left(\theta\right)}{p\left(\theta\right)}d\theta+\alpha\left(\int q(\theta)d\theta-1\right).
\]

Take derivative w.r.t. $q\left(\theta\right)$ and set it to $0$, we
obtain
\[
\lambda\mathcal{L}_{S}\left(\theta\right)+\log q\left(\theta\right)+1-\log p\left(\theta\right)+\alpha=0.
\]
\[
q\left(\theta\right)=\exp\left\{ -\lambda\mathcal{L}_{S}\left(\theta\right)\right\} p\left(\theta\right)\exp\left\{ -\alpha-1\right\} .
\]
\[
q\left(\theta\right)\propto\exp\left\{ -\lambda\mathcal{L}_{S}\left(\theta\right)\right\} p\left(\theta\right).
\]
\end{proof}

\begin{lemma}
\label{lem:modulus}Assume that the data space $\mathcal{X}$, the
label space $\mathcal{Y}$, and the model space $\Theta$ are compact
sets. There exist the modulus of continuity $\omega:\mathbb{R}^{+}\goto\mathbb{R}^{+}$
with $\lim_{t\goto0^{+}}\omega\left(t\right)=0$ such that $\left|\ell\left(f_{\theta}\left(x\right),y\right)-\ell\left(f_{\theta'}\left(x\right),y\right)\right|\leq\omega\left(\norm{\theta-\theta'}\right),\forall x\in\mathcal{X},y\in\mathcal{Y}$.
\end{lemma}
\begin{proof}
The loss function $\ell\left(f_{\theta}\left(x\right),y\right)$ is
continuous on the compact set $\mathcal{X}\times\mathcal{Y}\times\Theta$,
hence it is equip-continuous on this set. For every $\epsilon>0$,
there exists $\delta_{x},\delta_{y},\delta_{\theta}>0$ such that
\[
\forall\norm{x'-x}\leq\delta_{x},\norm{y'-y}\leq\delta_{y},\norm{\theta'-\theta}\leq\delta_{\theta},
\]
we have $\left|\ell\left(f_{\theta'}\left(x'\right),y'\right)-\ell\left(f_{\theta'}\left(x\right),y\right)\right|\leq\epsilon$.

Therefore, for all $\norm{\theta'-\theta}\leq\delta_{\theta}$, we
have
\[
\left|\ell\left(f_{\theta}\left(x\right),y\right)-\ell\left(f_{\theta'}\left(x\right),y\right)\right|\leq\epsilon,\forall x,y.
\]

This means that the family $\{\ell\left(f_{\theta}\left(x\right),y\right):x\in\mathcal{X},y\in\mathcal{Y}\}$
is equi-continuous w.r.t. $\theta\in\Theta$. This means the existence
of the common modulus of continuity $\omega:\mathbb{R}^{+}\goto\mathbb{R}^{+}$
with $\lim_{t\goto0^{+}}\omega\left(t\right)=0.$
\end{proof}
\begin{definition}
Given $\epsilon>0$, we say that $\Theta$ is $\epsilon$-covered
by a set $\Theta'$ if for all $\theta\in\Theta$, there exists $\theta'\in\Theta'$
such that $\norm{\theta'-\theta}\leq\epsilon.$ We define $\mathcal{N}\left(\Theta,\epsilon\right)$
as the cardinality set of the smallest set $\Theta'$ that covers
$\Theta$. 
\end{definition}
\begin{lemma}
\label{lem:coverage_number}Let $R=\max_{\theta\in\Theta}\norm{\theta}^{2}<\infty$
and $k$ is the dimension of $\Theta$. We can upper-bound the coverage
number as

\[
\mathcal{N}\left(\Theta,\epsilon\right)\leq\left(\frac{2R\sqrt{k}}{\epsilon}\right)^{k}.
\]
\end{lemma}
\begin{proof}
The proof can be found in Chapter 27 of \cite{shalev2014understanding}.
\end{proof}
By choosing $\epsilon=\frac{1}{n^{\frac{1}{2k}}}$, we obtain
\[
\mathcal{N}\left(\Theta,n^{-\frac{1}{2k}}\right)\leq\left(2R\sqrt{k}\right)^{k}\sqrt{n}.
\]

However, solving the optimization problem (OP) for the general data-label distribution $\mathcal{D}$ is generally intractable. To make it tractable, we find its upper-bound
which is relevant to the sharpness as shown in the following theorem.
\begin{theorem}
    \label{thm:appen_general_loss}
    (Theorem 3.2 in the main paper) Assume that $\Theta$ is a compact set. Given
any $\delta\in[0;1]$, with the probability at least $1-\delta$ over
the choice of $\mathcal{S}\sim\mathcal{D}^{n}$, for any distribution
$\mathbb{Q}$, we have
\begin{align*}
\mathcal{L}_{\mathcal{D}}\left(\mathbb{Q}\right) & \leq\mathbb{E}_{\theta\sim\mathbb{Q}}\left[\max_{\theta':\norm{\theta'-\theta}\leq\rho}\mathcal{L}_{\mathcal{S}}\left(\theta'\right)\right]+\mathcal{L}_{\mathcal{S}}\left(\mathbb{Q}\right)+\frac{1}{\sqrt{n}}+2\omega\left(\frac{1}{n^{\frac{1}{2k}}}\right)\\
 & +\sqrt{\frac{k\left(1+\log\left(1+\frac{2R^{2}}{\rho^{2}}\left(1+2\log\left(2R\sqrt{k}\right)+\frac{2}{k}\log n\right)\right)\right)+2\log\frac{n}{\delta}}{4(n-1)}},
\end{align*}
where we assume that $\mathcal{L}_{\mathcal{D}}\left(\mathbb{Q}\right)=\mathbb{E}_{\theta\sim\mathbb{Q}}\left[\mathcal{L}_{\mathcal{D}}\left(\theta\right)\right]\leq\mathbb{E}_{\theta\sim\mathbb{Q}}\left[\mathbb{E}_{\epsilon\sim\mathcal{N}(0,\sigma\mathbb{I})}\left[\mathcal{L}_{\mathcal{D}}\left(\theta+\epsilon\right)\right]\right]$
with $\sigma=\frac{\rho}{k^{1/2}\left(1+\sqrt{\frac{\log\left(N^{2}n\right)}{k}}\right)}$
and $N=\mathcal{N}\left(\Theta,n^{-\frac{1}{2k}}\right)$, $k$ is
the number of parameters of the models, $n=\left|S\right|$, $R=\max_{\theta\in\Theta}\norm{\theta}$,
and $\omega:\mathbb{R}^{+}\goto\mathbb{R}^{+}$ is a function such
that $\lim_{t\goto0^{+}}\omega\left(t\right)=0$.
\end{theorem}
\begin{proof}
Given $\epsilon=\frac{1}{n^{\frac{1}{2k}}}$, we denote $\Theta'=\left\{ \theta_{1}^{'},\dots,\theta_{N}^{'}\right\} $
where $N=\mathcal{N}\left(\Theta,n^{-\frac{1}{2k}}\right)\leq\left(2R\sqrt{k}\right)^{k}\sqrt{n}$
as the $\epsilon$-covered set of $\Theta$. We first examine a discrete
distribution 
\[
\mathbb{Q}=\sum_{i=1}^{m}\pi_{i}\delta_{\theta_{i}}.
\]

Without lossing the generalization, we can assume that $\norm{\theta_{i}^{'}-\theta_{i}}\leq\epsilon,\forall i=1,\dots,m$.
We note that $\theta_{1}^{'},\dots.\theta_{m}^{'}$ can be repeated
if $m>N$. Using Lemma \ref{lem:modulus}, let $\omega(\cdot)$ be
the modulus of continuity of $\ell\left(f_{\theta}\left(x\right),y\right)$
such that $\left|\ell\left(f_{\theta}\left(x\right),y\right)-\ell\left(f_{\theta}\left(x\right),y\right)\right|\leq\omega\left(\norm{\theta-\theta'}\right),\forall x,y$
and $\lim_{t\goto0}\omega\left(t\right)=0$. This implies that
\[
\left|\ell\left(f_{\theta_{i}}\left(x\right),y\right)-\ell\left(f_{\theta_{i}^{'}}\left(x\right),y\right)\right|\leq\omega\left(\epsilon\right)=\omega\left(\frac{1}{n^{\frac{1}{2k}}}\right),\forall x,y,i=1,\dots,m.
\]

We consider the distribution $\bar{\mathbb{Q}}=\sum_{i=1}^{m}\pi_{i}\mathcal{N}\left(\theta_{i}^{'},\sigma\mathbb{I}\right)$.
According to the McAllester PAC-Bayes bound, with the probability
$1-\delta$ over the choices of $\mathcal{S}\sim\mathcal{D}^{n}$,
for any distribution $\bar{\mathbb{P}}$, we have 
\[
\mathcal{L}_{\mathcal{D}}\left(\bar{\mathbb{Q}}\right)\leq\mathcal{L}_{S}\left(\bar{\mathbb{Q}}\right)+\sqrt{\frac{KL\left(\bar{\mathbb{Q}},\bar{\mathbb{P}}\right)+\log\frac{n}{\delta}}{2(n-1)}.}
\]

Let $\theta^{*}=\argmax{1\leq i\leq m}\norm{\theta_{i}^{'}}$. We
consider the distribution $\bar{\mathbb{P}}=\mathcal{N}\left(0,\sigma_{\mathbb{P}}\right)$where
$\sigma_{\mathbb{P}}^{2}=c\exp\left\{ \frac{1-j}{k}\right\} $ with
$c=\sigma^{2}\left(1+\exp\left\{ \frac{4n}{k}\right\} \right)$ and
$j=\left\lfloor 1+k\log\frac{c}{\sigma^{2}+\frac{\norm{\theta^{*}}^{2}}{k}}\right\rfloor =\left\lfloor 1+k\log\frac{\sigma^{2}\left(1+\exp\left\{ \frac{4n}{k}\right\} \right)}{\sigma^{2}+\frac{\norm{\theta^{*}}^{2}}{k}}\right\rfloor $.
It follows that
\[
\sigma^{2}+\frac{\norm{\theta^{*}}^{2}}{k}\leq\sigma_{\mathbb{P}}\leq\exp\left\{ \frac{1}{k}\right\} \left(\sigma^{2}+\frac{\norm{\theta^{*}}^{2}}{k}\right).
\]

We have 
\[
KL\left(\mathcal{N}\left(\theta_{i}^{'},\sigma\mathbb{I}\right),\bar{\mathbb{P}}\right)=\frac{1}{2}\left[\frac{k\sigma^{2}+\norm{\theta_{i}^{'}}^{2}}{\sigma_{\mathbb{P}}^{2}}-k+k\log\left(\frac{\sigma_{\mathbb{P}}^{2}}{\sigma^{2}}\right)\right].
\]
\[
KL\left(\mathcal{N}\left(\theta^{*},\sigma\mathbb{I}\right),\bar{\mathbb{P}}\right)=\max_{i}KL\left(\mathcal{N}\left(\theta_{i}^{'},\sigma\mathbb{I}\right),\bar{\mathbb{P}}\right).
\]
\[
KL\left(\bar{\mathbb{Q}},\bar{\mathbb{P}}\right)\leq\sum_{i=1}^{m}\pi_{i}KL\left(\mathcal{N}\left(\theta_{i}^{'},\sigma\mathbb{I}\right),\bar{\mathbb{P}}\right)\leq KL\left(\mathcal{N}\left(\theta^{*},\sigma\mathbb{I}\right),\bar{\mathbb{P}}\right).
\]

We now bound $KL\left(\mathcal{N}\left(\theta^{*},\sigma\mathbb{I}\right),\bar{\mathbb{P}}\right)$
\begin{align*}
KL\left(\mathcal{N}\left(\theta^{*},\sigma\mathbb{I}\right),\bar{\mathbb{P}}\right) & =\frac{1}{2}\left[\frac{k\sigma^{2}+\norm{\theta^{*}}^{2}}{\sigma_{\mathbb{P}}^{2}}-k+k\log\left(\frac{\sigma_{\mathbb{P}}^{2}}{\sigma^{2}}\right)\right]\\
\leq & \frac{1}{2}\left[\frac{k\sigma^{2}+\norm{\theta^{*}}^{2}}{\sigma^{2}+\frac{\norm{\theta^{*}}^{2}}{k}}-k+k\log\left(\frac{\exp\left\{ \frac{1}{k}\right\} \left(\sigma^{2}+\frac{\norm{\theta^{*}}^{2}}{k}\right)}{\sigma^{2}}\right)\right]\\
\leq & \frac{k}{2}\left(1+\log\left(1+\frac{\norm{\theta^{*}}^{2}}{k\sigma^{2}}\right)\right).
\end{align*}

Therefore, with the probability $1-\delta$, we reach
\[
\mathcal{L}_{\mathcal{D}}\left(\bar{\mathbb{Q}}\right)\leq\mathcal{L}_{S}\left(\bar{\mathbb{Q}}\right)+\sqrt{\frac{k\left(1+\log\left(1+\frac{\norm{\theta^{*}}^{2}}{k\sigma^{2}}\right)\right)+2\log\frac{n}{\delta}}{4(n-1)}.}
\]
\begin{align*}
\mathbb{E}_{\theta\sim\sum_{i=1}^{m}\pi_{i}\mathcal{N}\left(\theta_{i}^{'},\sigma\mathbb{I}\right)}\left[\mathcal{L}_{\mathcal{D}}\left(\theta\right)\right] & \leq\mathbb{E}_{\theta\sim\sum_{i=1}^{m}\pi_{i}\mathcal{N}\left(\theta_{i}^{'},\sigma\mathbb{I}\right)}\left[\mathcal{L}_{\mathcal{S}}\left(\theta\right)\right]+\sqrt{\frac{k\left(1+\log\left(1+\frac{\norm{\theta^{*}}^{2}}{k\sigma^{2}}\right)\right)+2\log\frac{n}{\delta}}{4(n-1)}.}\\
\leq & \sum_{i=1}^{m}\pi_{i}\mathbb{E}_{\epsilon_{i}\sim\mathcal{N}\left(0,\sigma\mathbb{I}\right)}\left[\mathcal{L}_{\mathcal{S}}\left(\theta_{i}^{'}+\epsilon_{i}\right)\right]+\sqrt{\frac{k\left(1+\log\left(1+\frac{\norm{\theta^{*}}^{2}}{k\sigma^{2}}\right)\right)+2\log\frac{n}{\delta}}{4(n-1)}.}
\end{align*}

Note that 
\begin{align*}
\mathbb{E}_{\theta\sim\mathcal{N}\left(\theta_{i},\sigma\mathbb{I}\right)}\left[\mathcal{L}_{\mathcal{D}}\left(\theta\right)\right]-\mathbb{E}_{\theta\sim\mathcal{N}\left(\theta_{i}^{'},\sigma\mathbb{I}\right)}\left[\mathcal{L}_{\mathcal{D}}\left(\theta\right)\right] & =\int\left[\mathcal{L}_{\mathcal{D}}\left(\theta_{i}+\epsilon_{i}\right)-\mathcal{L}_{\mathcal{D}}\left(\theta_{i}^{'}+\epsilon_{i}\right)\right]\mathcal{N}\left(\epsilon_{i}\mid0,\sigma I\right)d\epsilon_{i}\\
\leq & \int\omega\left(\frac{1}{n^{\frac{1}{2k}}}\right)\mathcal{N}\left(\epsilon_{i}\mid0,\sigma I\right)d\epsilon_{i}=\omega\left(\frac{1}{n^{\frac{1}{2k}}}\right).
\end{align*}
\[
\mathbb{E}_{\theta\sim\mathcal{N}\left(\theta_{i},\sigma\mathbb{I}\right)}\left[\mathcal{L}_{\mathcal{D}}\left(\theta\right)\right]\leq\mathbb{E}_{\theta\sim\mathcal{N}\left(\theta_{i}^{'},\sigma\mathbb{I}\right)}\left[\mathcal{L}_{\mathcal{D}}\left(\theta\right)\right]+\omega\left(\frac{1}{n^{\frac{1}{2k}}}\right).
\]
\[
\sum_{i=1}^{m}\pi_{i}\mathbb{E}_{\theta\sim\mathcal{N}\left(\theta_{i},\sigma\mathbb{I}\right)}\left[\mathcal{L}_{\mathcal{D}}\left(\theta\right)\right]\leq\sum_{i=1}^{m}\pi_{i}\mathbb{E}_{\theta\sim\mathcal{N}\left(\theta_{i}^{'},\sigma\mathbb{I}\right)}\left[\mathcal{L}_{\mathcal{D}}\left(\theta\right)\right]+\omega\left(\frac{1}{n^{\frac{1}{2k}}}\right),
\]
therefore we have
\begin{align*}
\mathbb{E}_{\theta\sim\sum_{i=1}^{m}\pi_{i}\mathcal{N}\left(\theta_{i},\sigma\mathbb{I}\right)}\left[\mathcal{L}_{\mathcal{D}}\left(\theta\right)\right] & \leq\sum_{i=1}^{m}\pi_{i}\mathbb{E}_{\epsilon_{i}\sim\mathcal{N}\left(0,\sigma\mathbb{I}\right)}\left[\mathcal{L}_{\mathcal{S}}\left(\theta_{i}^{'}+\epsilon_{i}\right)\right]\\
 & +\sqrt{\frac{k\left(1+\log\left(1+\frac{\norm{\theta^{*}}^{2}}{k\sigma^{2}}\right)\right)+2\log\frac{n}{\delta}}{4(n-1)}}+\omega\left(\frac{1}{n^{\frac{1}{2k}}}\right).
\end{align*}

Using the assumption 
\[
\mathcal{L}_{\mathcal{D}}\left(\mathbb{Q}\right)=\mathbb{E}_{\theta\sim\mathbb{Q}}\left[\mathcal{L}_{\mathcal{D}}\left(\theta\right)\right]\leq\mathbb{E}_{\theta\sim\mathbb{Q}}\left[\mathbb{E}_{\epsilon\sim\mathcal{N}\left(0,\sigma\mathbb{I}\right)}\left[\mathcal{L}_{\mathcal{D}}\left(\theta+\epsilon\right)\right]\right]=\mathbb{E}_{\theta\sim\sum_{i=1}^{m}\pi_{i}\mathcal{N}\left(\theta_{i},\sigma\mathbb{I}\right)}\left[\mathcal{L}_{\mathcal{D}}\left(\theta\right)\right],
\]
we obtain
\begin{align*}
\mathcal{L}_{\mathcal{D}}\left(\mathbb{Q}\right) & \leq\sum_{i=1}^{m}\pi_{i}\mathbb{E}_{\epsilon_{i}\sim\mathcal{N}\left(0,\sigma\mathbb{I}\right)}\left[\mathcal{L}_{\mathcal{S}}\left(\theta_{i}+\epsilon_{i}\right)\right]\\
 & +\sqrt{\frac{k\left(1+\log\left(1+\frac{R^{2}}{k\sigma^{2}}\right)\right)+2\log\frac{n}{\delta}}{4(n-1)}}+\omega\left(\frac{1}{n^{\frac{1}{2k}}}\right).
\end{align*}

Because $\epsilon_{i}\sim\mathcal{N}\left(0,\sigma\mathbb{I}\right)$,
$\norm{\epsilon_{i}}^{2}$ follows the Chi-squared distribution. Therefore,
we have for any $i\in[m]$
\[
\mathbb{P}\left(\norm{\epsilon_{i}}^{2}-k\sigma^{2}\geq2\sigma^{2}\sqrt{kt}+2t\sigma^{2}\right)\leq\exp(-t),\forall t.
\]
\[
\mathbb{P}\left(\max_{i\in[m]}\norm{\epsilon_{i}}^{2}-k\sigma^{2}\geq2\sigma^{2}\sqrt{kt}+2t\sigma^{2}\right)\leq N\exp(-t),\forall t.,
\]
since the cardinality of $\left|\{\theta_{1}^{'},\dots.,\theta_{m}^{'}\}\right|$
cannot exceed $N$.
\[
\mathbb{P}\left(\max_{i\in[m]}\norm{\epsilon_{i}}^{2}-k\sigma^{2}<2\sigma^{2}\sqrt{kt}+2t\sigma^{2}\right)>1-N\exp(-t),\forall t.
\]

By choosing $t=\log\left(Nn^{1/2}\right),$ with the probability at
least $1-\frac{1}{\sqrt{n}}$, we have for all $i\in[m]$

\[
\norm{\epsilon_{i}}^{2}<\sigma^{2}k\left(1+\frac{\log\left(N^{2}n\right)}{k}+2\sqrt{\frac{\log\left(Nn^{1/2}\right)}{k}}\right)\leq\sigma^{2}k\left(1+\sqrt{\frac{\log\left(N^{2}n\right)}{k}}\right)^{2}.
\]

By choosing $\sigma=\frac{\rho}{k^{1/2}\left(1+\sqrt{\frac{\log\left(N^{2}n\right)}{k}}\right)}$,
with the probability at least $1-\frac{1}{\sqrt{n}}$, we have for
all $i\in[m]$
\[
\norm{\epsilon_{i}}<\rho.
\]

We now derive 
\begin{align*}
\mathcal{L}_{\mathcal{D}}\left(\mathbb{Q}\right) & \leq\sum_{i=1}^{m}\pi_{i}\left(1-\frac{1}{\sqrt{n}}\right)\max_{\norm{\epsilon_{i}}\leq\rho}\mathcal{L}_{\mathcal{S}}\left(\theta_{i}^{'}+\epsilon_{i}\right)\\
&+\frac{1}{\sqrt{n}}+\sqrt{\frac{k\left(1+\log\left(1+\frac{R^{2}}{k\sigma^{2}}\right)\right)+2\log\frac{n}{\delta}}{4(n-1)}}+\omega\left(\frac{1}{n^{\frac{1}{2k}}}\right)\\
 & \leq\left(1-\frac{1}{\sqrt{n}}\right)\sum_{i=1}^{m}\pi_{i}\max_{\norm{\epsilon_{i}}\leq\rho}\mathcal{L}_{\mathcal{S}}\left(\theta_{i}^{'}+\epsilon_{i}\right)+\frac{1}{\sqrt{n}}\\
 & +\sqrt{\frac{k\left(1+\log\left(1+\frac{R^{2}}{\rho^{2}}\left(1+\sqrt{\frac{\log\left(N^{2}n\right)}{k}}\right)^{2}\right)\right)+2\log\frac{n}{\delta}}{4(n-1)}}+\omega\left(\frac{1}{n^{\frac{1}{2k}}}\right)\\
 & \leq\sum_{i=1}^{m}\pi_{i}\max_{\norm{\epsilon_{i}}\leq\rho}\mathcal{L}_{\mathcal{S}}\left(\theta_{i}^{'}+\epsilon_{i}\right)+\frac{1}{\sqrt{n}}\\
 & +\sqrt{\frac{k\left(1+\log\left(1+\frac{2R^{2}}{\rho^{2}}\left(1+\frac{\log\left(N^{2}n\right)}{k}\right)\right)\right)+2\log\frac{n}{\delta}}{4(n-1)}}+\omega\left(\frac{1}{n^{\frac{1}{2k}}}\right)\\
 & \leq\sum_{i=1}^{m}\pi_{i}\max_{\norm{\epsilon_{i}}\leq\rho}\mathcal{L}_{\mathcal{S}}\left(\theta_{i}^{'}+\epsilon_{i}\right)+\frac{1}{\sqrt{n}}\\
 & +\sqrt{\frac{k\left(1+\log\left(1+\frac{2R^{2}}{\rho^{2}}\left(1+2\log\left(2R\sqrt{k}\right)+\frac{2}{k}\log n\right)\right)\right)+2\log\frac{n}{\delta}}{4(n-1)}}+\omega\left(\frac{1}{n^{\frac{1}{2k}}}\right).
\end{align*}


Note that for all $i\in[m]$
\[
\max_{\norm{\epsilon_{i}}\leq\rho}\mathcal{L}_{\mathcal{S}}\left(\theta_{i}^{'}+\epsilon_{i}\right)\leq\max_{\norm{\epsilon_{i}}\leq\rho}\mathcal{L}_{\mathcal{S}}\left(\theta_{i}+\epsilon_{i}\right)+\omega\left(\frac{1}{n^{\frac{1}{2k}}}\right),
\]
therefore, we reach
\begin{align*}
\mathcal{L}_{\mathcal{D}}\left(\mathbb{Q}\right) & \leq\sum_{i=1}^{m}\pi_{i}\max_{\norm{\epsilon_{i}}\leq\rho}\mathcal{L}_{\mathcal{S}}\left(\theta_{i}+\epsilon_{i}\right)+\frac{1}{\sqrt{n}}\\
 & +\sqrt{\frac{k\left(1+\log\left(1+\frac{2R^{2}}{\rho^{2}}\left(1+2\log\left(2R\sqrt{k}\right)+\frac{2}{k}\log n\right)\right)\right)+2\log\frac{n}{\delta}}{4(n-1)}}+2\omega\left(\frac{1}{n^{\frac{1}{2k}}}\right)\\
 & \leq\mathcal{L}_{\mathcal{S}}\left(\mathbb{Q}\right)+\frac{1}{\sqrt{n}}+2\omega\left(\frac{1}{n^{\frac{1}{2k}}}\right)\\
 &+\sqrt{\frac{k\left(1+\log\left(1+\frac{2R^{2}}{\rho^{2}}\left(1+2\log\left(2R\sqrt{k}\right)+\frac{2}{k}\log n\right)\right)\right)+2\log\frac{n}{\delta}}{4(n-1)}}.
\end{align*}

For any distribution $\mathbb{Q}$, we approximate $\mathbb{Q}$ by
its empirical distribution
\[
\mathbb{Q}_{m}=\frac{1}{m}\sum_{i=1}^{m}\delta_{\theta_{i}},
\]
which weakly converges to $\mathbb{Q}$ when $m\goto\infty$. By using
the achieved results for $\mathbb{Q}_{m}$ and taking limitation when
$m\goto\infty$, we reach the conclusion. 
\end{proof}

\begin{theorem}
    \label{thm:appen_Bayses_SAM_solution}(Theorem 3.3 in the main paper) The optimal solution the OP in 
is the sharpness-aware posterior distribution $\mathbb{Q}_{S}^{SA}$
with the density function $q^{SA}(\theta|\mathcal{S})$:
\begin{align*}
q^{SA}(\theta|\mathcal{S}) & \propto\exp\left\{ -\lambda\max_{\theta':\norm{\theta'-\theta}\leq\rho}\mathcal{L}_{\mathcal{S}}\left(\theta'\right)\right\} p\left(\theta\right)\\
 & =\exp\left\{ -\lambda\mathcal{L}_{\mathcal{S}}\left(s\left(\theta\right)\right)\right\} p\left(\theta\right),
\end{align*}
where we have defined $s\left(\theta\right)=\argmax{\theta':\norm{\theta'-\theta}\leq\rho}\mathcal{L}_{\mathcal{S}}\left(\theta'\right)$.
\end{theorem}
\begin{proof}
We have
\[
\lambda\mathbb{E}_{\theta\sim\mathbb{Q}}\left[\max_{\theta':\norm{\theta'-\theta}\leq\rho}\mathcal{L}_{\mathcal{S}}\left(\theta'\right)\right]+KL\left(\mathbb{Q},\mathbb{P}\right)=\lambda\int\mathcal{L}_{S}\left(s\left(\theta\right)\right)q\left(\theta\right)d\theta+\int q(\theta)\log\frac{q\left(\theta\right)}{p\left(\theta\right)}d\theta.
\]

The Lagrange function is as follows
\[
L\left(q,\alpha\right)=\lambda\int\mathcal{L}_{S}\left(s(\theta)\right)q\left(\theta\right)d\theta+\int q(\theta)\log\frac{q\left(\theta\right)}{p\left(\theta\right)}d\theta+\alpha\left(\int q(\theta)d\theta-1\right).
\]

Take derivative w.r.t. $q\left(\theta\right)$ and set it to $0$, we
obtain
\[
\lambda\mathcal{L}_{S}\left(s(\theta)\right)+\log q\left(\theta\right)+1-\log p\left(\theta\right)+\alpha=0.
\]
\[
q\left(\theta\right)=\exp\left\{ -\lambda\mathcal{L}_{S}\left(s(\theta)\right)\right\} p\left(\theta\right)\exp\left\{ -\alpha-1\right\} .
\]
\[
q\left(\theta\right)\propto\exp\left\{ -\lambda\mathcal{L}_{S}\left(s(\theta)\right)\right\} p\left(\theta\right).
\]
\end{proof}

\subsection{Technicalities of the baselines and the corresponding flat versions}
\label{sec:train_tech}
In what follows, we present how the baselines used in the experiments can be viewed as variational and MCMC approaches and incorporate our sharpness-aware technique.

\paragraph{Bayesian deep ensemble \cite{lakshminarayanan2017simple}:} We consider the approximate posterior $q_\phi = \frac{1}{K}\sum_{k=1}^K \delta_{\theta_k}$ where $\delta$ is the Dirac delta distribution as a uniform distribution over several base models $\theta_{1:K}$. Considering the prior distribution $p(\theta) = \mathcal{N}(0, \mathbb{I})$, we have the following OPs for the non-flat and flat versions.

\textbf{Non-flat version:}
\[
\min_{\theta_{1:K}}\left\{ \mathbb{E}_{\theta_{k}\sim q_{\phi}}\left[\lambda\mathcal{L}_{S}\left(\theta_{k}\right)\right]+KL\left(\frac{1}{K}\sum_{k=1}^{K}\delta_{\theta_{k}},\mathcal{N}\left(0,\mathbb{I}\right)\right)\right\} ,
\]
where $KL\left(\frac{1}{K}\sum_{k=1}^{K}\delta_{\theta_{k}},\mathcal{N}\left(0,\mathbb{I}\right)\right)=-\frac{1}{K}\sum_{k=1}^{K}\log\mathcal{N}\left(\theta_{k}\mid0,\mathbb{I}\right)+\text{const}$, leading to the L2 regularization terms.

\textbf{Flat version:}
\[
\min_{\theta_{1:K}}\left\{ \mathbb{E}_{\theta_{k}\sim q_{\phi}}\left[\lambda\max_{\theta':\Vert\theta'-\theta_{k}\Vert\leq\rho}\mathcal{L}_{S}\left(\theta'\right)\right]+KL\left(\frac{1}{K}\sum_{k=1}^{K}\delta_{\theta_{k}},\mathcal{N}\left(0,\mathbb{I}\right)\right)\right\}.
\]

\paragraph{MC-Dropout \cite{gal2016-mcdropout}:}
As shown in \cite{gal2016-mcdropout}, the MC-dropout can be viewed as a BNN with the approximate posterior $q_\phi = \delta_\phi$ where $\phi$ is a fully-connected base model without any dropout and the prior distribution $p(\theta) = \mathcal{N}(0, \mathbb{I})$. The $KL(q_\phi, p(\theta))$ can be approximated which turns out to be a weighted L2 regularization where the weights are proportional to the keep-prob rates at the layers. The main term $\mathbb{E}_{\theta\sim q_{\phi}}\left[\lambda\mathcal{L}_{S}\left(\theta\right)\right]$ can be interpreted as applying the dropout before minimizing the loss. For our flat version, the main term is $\mathbb{E}_{\theta\sim q_{\phi}}\left[\lambda\max_{\theta':\Vert\theta'-\theta\Vert\leq\rho}\mathcal{L}_{S}\left(\theta'\right)\right]$.  

 \paragraph{BNNs with Stochastic Gradient Langevin Dynamics
(SGLD) \cite{welling2011bayesian}:} For SGLD, we sample one or several particle models directly from the posterior distribution $q(\theta \mid S)$ for the non-flat version and from the SA-posterior distribution $q^{SA}(\theta \mid S)$ for the flat version. For the non-flat version, the update is similar to the mini-batch SGD except that we add small Gaussian noises to the particle models. For our flat version, we first compute the perturbed model $\theta^a$ for a given particle model $\theta$ and use the mini-batch SGD update with the gradient evaluated at $\theta^a$ together with small Gaussian noises. 

\paragraph{SWAG \cite{Maddox2019-swag}:} We consider SWAG as an MCMC approach, where we keep a trajectory of particle models using SWA. Additionally, the covariance matrices are determined based on this trajectory to form an approximate Gaussian posterior. In the corresponding flat version to this approach, we employ SWA to sample from the SA (Sharpness-Aware) posterior. Specifically, we first calculate the perturbed model $\theta^a$ based on the current model $\theta$ and then employ mini-batch SGD updates with the gradient evaluated at model $\theta^a$. Finally, we update the final model using the SWA strategy.

\section{Additional experiments}

\subsection{Comparison with bSAM method}
We conduct experiments to compare our flat BNN with bSAM \cite{mollenhoff2022sam} on Resnet18, the results are shown in Table \ref{tab:ex-bsam}. The authors of bSAM explored the relationship between SAM and BNN and proposed a combination of SAM and Adam to optimize the mean of parameters in BNN networks while keeping the variance fixed. The results clearly indicate that our flat BNN outperforms bSAM in most metric scores. Here we note that we are unable to evaluate bSAM on the architectures used in Tables 1 and 2 in the main paper because the authors did not release the code. Instead, we run our methods with the setting mentioned in the bSAM paper.

\begin{table*}[h]
\centering
\caption{Classification score on Resnet18}
\label{tab:ex-bsam}
\resizebox{.65\columnwidth}{!}{
\begin{tabular}{lccc|ccc}
\midrule

  & \multicolumn{3}{c}{\textbf{CIFAR-10}}      & \multicolumn{3}{c}{\textbf{CIFAR-100}} \\
Method     & ACC $\uparrow$   & NLL $\downarrow$  & ECE $\downarrow$  & ACC $\uparrow$   & NLL $\downarrow$      & ECE $\downarrow$  \\ \hline \midrule
bSAM & 96.15 & 0.1200 & 0.0049 & 80.22 & \textbf{0.7000} & 0.0310 \\
F-SWAG-Diag & 96.56 & 0.1047 & \textbf{0.0037} & 80.70 & 0.7012 & \textbf{0.0227} \\
F-SWAG & \textbf{96.58} & \textbf{0.1045} & 0.0045 & \textbf{80.74} & 0.7024 & 0.0243  \\ \bottomrule
\end{tabular}}
\end{table*}

\begin{table}[t]
\centering
\caption{Classification score on CIFAR-10-C using PreResNet-164 model when training with CIFAR-10 dataset}
\label{tab:ex-acc-CIFAR-10c-full}
\resizebox{1.\columnwidth}{!}{%
\begin{tabular}{lcc|ccccc|cc}
\toprule
& \multicolumn{4}{c}{\textbf{ECE $\downarrow$}} &  \multicolumn{4}{c}{\textbf{Accuracy $\uparrow$}} \\
Method &  SWAG-D & F-SWAG-D & SWAG & F-SWAG & 	& 		SWAG-D & F-SWAG-D & SWAG & F-SWAG \\ \midrule \midrule
Gaussian noise &  0.0765 & 0.0765 & 0.1032 & 0.0091 & 	& 		72.01 & 73.95 & 71.58 & 73.43 \\
Shot noise &  0.0661 & 0.0647 & 0.0892 & 0.0075 & 	& 		75.77 & 77.09 & 75.44 & 76.28 \\
Speckle noise &  0.0711 & 0.0686 & 0.0921 & 0.0072 & 	& 		75.55 & 76.71 & 75.36 & 76.15 \\
Impulse noise &  0.0779 & 0.0706 & 0.0988 & 0.0077 & 	& 		73.74 & 74.61 & 73.71 & 74.49 \\
Defocus blur &  0.0108 & 0.0071 & 0.0178 & 0.0256 & 	& 		92.16 & 91.55 & 92.14 & 91.63 \\
Gaussian blur &  0.0130 & 0.0116 & 0.0214 & 0.0239 & 	& 		90.79 & 89.73 & 90.67 & 90.22 \\
Motion blur &  0.0147 & 0.0103 & 0.0233 & 0.0298 & 	& 		90.33 & 90.22 & 90.20 & 90.52 \\
Zoom blur &  0.0099 & 0.0070 & 0.0185 & 0.0301 & 	& 		91.24 & 90.71 & 91.12 & 91.36 \\
snow &  0.0298 & 0.0245 & 0.0419 & 0.0208 & 	& 		86.17 & 86.11 & 86.10 & 85.81 \\
Fog &  0.0114 & 0.0075 & 0.0176 & 0.0259 & 	& 		91.67 & 91.22 & 91.64 & 91.27 \\
Brightness &  0.0081 & 0.0076 & 0.0129 & 0.0281 & 	& 		93.47 & 92.94 & 93.45 & 92.98 \\
Contrast &  0.0110 & 0.0127 & 0.0141 & 0.0306 & 	& 		91.34 & 90.59 & 91.36 & 90.73 \\
Elastic transform &  0.0244 & 0.0213 & 0.0367 & 0.0220 & 	& 		87.06 & 86.61 & 86.98 & 86.94 \\
Pixelate &  0.0350 & 0.0269 & 0.0463 & 0.0124 & 	& 		85.75 & 86.11 & 85.61 & 85.74 \\
Jpeg compression &  0.0605 & 0.0522 & 0.0813 & 0.0093 & 	& 		78.01 & 78.90 & 77.57 & 79.80 \\
Spatter &  0.0242 & 0.0163 & 0.0341 & 0.0227 & 	& 		87.94 & 87.61 & 87.99 & 88.13 \\
Saturate &  0.0112 & 0.0080 & 0.0179 & 0.0288 & 	& 		92.10 & 91.79 & 92.09 & 91.81 \\
Frost &  0.0253 & 0.0174 & 0.0365 & 0.0215 &	& 		86.60 & 86.48 & 86.52 & 86.39 \\ \midrule
\textbf{Average} &  0.0322 & \textbf{0.0283} & 0.0446 & \textbf{0.0201} & & 			85.65 & \textbf{85.71} & 85.52 & \textbf{85.76} \\ \bottomrule
\end{tabular}
}
\end{table}

\subsection{Full result of Out-of-distribution prediction}
In Section 4.2 of the main paper, we provide a comprehensive analysis of the performance concerning various corruption groups, including noise, blur, weather conditions, and digital distortions. We present the detailed results for each corruption type in Table \ref{tab:ex-acc-CIFAR-10c-full}, providing a deeper understanding of the impact of these corruptions on the model's performance. On average, flat BNNs outperform their non-flat counterparts, especially on ECE with a notable margin. These findings further emphasize the effectiveness of flat BNNs in enhancing robustness and generalization against various corruptions.

\subsection{Additional ablation studies}

\textbf{Comparison of Hessian eigenvalue}
We report the log scale of the largest eigenvalue of the Hessian matrix over several methods applying to WideResNet28x10 using CIFAR-100, and the ratio of the largest and fifth eigenvalue as shown in Table \ref{tab:hessian_cost}, which evidently indicates that our method updates models to minima having lower curvature.

\begin{table}[h]
\centering
\caption{Log scale of Hessian eigenvalue of WideResNet28x10 training on CIFAR-100. $\lambda_1$ is the largest eigenvalue and $\lambda_5$ is 5th largest eigenvalue}
\label{tab:hessian_cost}
\begin{tabular}{ccccc}

\multicolumn{1}{c}{Method} & $\lambda_1$ $\downarrow$   & $\lambda_1$/$\lambda_5$ \\ \midrule \midrule
SWAG                       &    4.17 $\pm$ 0.001    & 1.17  $\pm$ 0.012     \\
F-SWAG                    &   \textbf{4.08 $\pm$ 0.000}    & 1.17 $\pm$ 0.020    \\  \midrule
SGLD                       & 3.34 $\pm$ 0.031       & 1.17  $\pm$ 0.009    \\
F-SGLD                     & \textbf{2.83$\pm$ 0.029}        & 1.15   $\pm$ 0.010     \\  \midrule
Deep-ensemble              & 4.64 $\pm$ 0.055 & 1.45 $\pm$ 0.020   \\
F-Deep-ensemble            & \textbf{4.01 $\pm$ 0.054} & 1.58 $\pm$ 0.032   \\ \midrule
\end{tabular}
\end{table}

\textbf{Computational cost}
Our flat-seeking method requires the computation of gradients twice: initially to obtain the perturbed model $\theta'$ and subsequently to update the model. Consequently, the training time is nearly double in comparison to non-flat counterparts, as indicated in Table \ref{tab:comp_cost}. Note that the Deep Ensemble settings utilize multiple models training individually for prediction and we report training time for one model in each setting.

\begin{table}[h!]
\centering
\caption{Comparison of training time per epoch}
\label{tab:comp_cost}
\resizebox{1.\columnwidth}{!}{%
\begin{tabular}{c|cc|cc|cc}
Network \& Dataset & SWAG    & F-SWAG & SGLD & F-SGLD & Deep-ensemble & F-Deep-ensemble \\ \midrule \midrule
WideResNet28x10 \& CIFAR-100 & 110s & 169s &  110s & 233s &  110s & 218s \\ \midrule
Densenet-161 \& ImageNet & 1.75h & 2.28h &  1.78h & 2.49h & -  & - \\
ResNet-152 \& ImageNet & 1.59h & 2.15h &  1.64h & 2.22h &  - & - \\ \bottomrule
\end{tabular}
}
\end{table}

\textbf{The effect of $KL$ term in Deep-ensemble settings} We present the results of training Deep-ensemble with SAM following the formula for the flat version in Section \ref{sec:train_tech} but without KL (or L2 regularisation) in Table \ref{tab:deep_ensem}. Each experiment is performed three times and reports the mean and standard deviation. Based on the result, without KL loss, our method still manages to yield better numbers than the non-flat counterparts.

\begin{table}[h]
\centering
\caption{Experiments of F-Deep-ensemble variations on CIFAR-100 dataset using WideResNet28x10. Each experiment is conducted with three different random seeds to calculate mean and standard deviation}
\label{tab:deep_ensem}
\resizebox{.8\columnwidth}{!}{
\begin{tabular}{lccc}
\toprule

  & \multicolumn{3}{c}{\textbf{WideResNet28x10}}   \\
Model     & ACC $\uparrow$   & NLL $\downarrow$  & ECE $\downarrow$ \\ \midrule \midrule

Deep-ensemble   &  83.04 $\pm$ 0.15 &  0.6958 $\pm$ 0.0335   &  0.0483 $\pm$ 0.0017  \\  \midrule
F-Deep-ensemble \textbf{(Our)}  &  \textbf{84.52 $\pm$ 0.03} & \textbf{0.5644 $\pm$ 0.0106} & \textbf{0.0191 $\pm$ 0.0039}    \\
F-Deep-ensemble (w/o L2) &  83.80 $\pm$  0.10  &  0.7026 $\pm$ 0.0007   &  0.0594 $\pm$ 0.0005 \\
\bottomrule
\end{tabular}}
\end{table}

\section{Experimental settings}
\textbf{CIFAR:} We conduct experiments using PreResNet-164, WideResNet28x10, Resnet10 and Resnet18 on both CIFAR-10 and CIFAR-100. The total number of images in these datasets is 60,000, which comprises 50,000 instances for training and 10,000 for testing. For each network-dataset pair, we apply Sharpness-Aware Bayesian methodology to various settings, including F-SGLD, F-SGVB, F-SWAG-Diag, F-SWAG, F-MC-Dropout, and F-Deep-Ensemble.

In the experiments presented in Tables 1 and 2 in the main paper, we train all models for 300 epochs using SGD, with a learning rate of 0.1 and a cosine schedule. We start collecting models after epoch 161 for the F-SWA and F-SWAG settings, consistent with the protocol in \cite{Maddox2019-swag}. Additionally, we set $\rho=0.05$ for CIFAR-10 and $\rho=0.1$ for CIFAR-100 in all experiments, except for Resnet10 and Resnet18, where $\rho$ is set to 0.01. The training set is augmented with basic data augmentations, including horizontal flip, padding by four pixels, random crop, and normalization. For the experiments presented in Table 3 in the main paper, we apply the same augmentations to the training set as in the experiments in Tables 1 and 2. However, the models are trained for 200 epochs using the Adam optimizer, with a learning rate of 0.001 and a plateau schedule. It's worth noting that SGVB and SGVB-LRT perform poorly with other settings than those mentioned, making it challenging to scale up this approach.

For the baseline of the Deep-Ensemble, SGLD, SGVB and SGVB-LRT methods, we reproduce results following the hyper-parameters and processes as our flat versions. Note that we train three independent models for the Deep-Ensemble method. For inference, we do an ensemble on 30 sample models for all settings sampled from posterior distributions. To ensure the stability of the method, we repeat each set three times with different random seeds and report the mean and standard deviation.

\textbf{ImageNet:} This is a large and challenging dataset with 1000 classes. We conduct experiments with Densenet-161 and ResNet-152 architecture on F-SWAG-Diag, F-SWAG, and F-SGLD. For all settings, we initialize the models with pre-trained weights on the ImageNet dataset, obtained from the \textit{torchvision} package, then fine-tuned for 10 epochs with $\rho=0.05$. We start collecting 4 models per epoch at the beginning of the fine-tuning process and evaluate them following a protocol consistent with the CIFAR dataset experiments.


The performance metrics for the SWAG-Diag, SWAG, and MC-Dropout methods are sourced from the original paper by Maddox et al. \cite{Maddox2019-swag}, except for the MC-Dropout result on PreResNet-164 for the CIFAR-100 dataset, which we reproduce due to its unavailability. The performance of bSAM is taken from \cite{mollenhoff2022sam}. 

It's important to note that the purpose of these experiments was not to achieve state-of-the-art performance. Instead, we aim to demonstrate the utility of the sharpness-aware posterior when integrated with specific Bayesian Neural Networks. The implementation is provided in \url{https://github.com/anh-ntv/flat_bnn.git}.

\textbf{Hyper-parameters for training}: Table \ref{tab:hyperparam} provides our setup for both training and testing phases. Note that the SWAG-Diag method follows the same setup as SWAG. Typically, using the default $\rho = 0.05$ yields a good performance across all experiments. However, $\rho=0.1$ is recommended for the CIFAR-100 dataset in \cite{foret2021sharpnessaware}. For model evaluation, we use the checkpoint from the final epoch without taking into account the validation set's performance.

\begin{table}
\centering
\caption{Hyperparameters for training both flat and non-flat versions of BNNs. All models are trained with the input resolution of $224 \times 224$ and cosine learning rate decay, except experiments of SGVB and SGVB-LRT, which use an input resolution of $32 \times 32$ }
\label{tab:hyperparam}
\resizebox{1.\columnwidth}{!}{\begin{tabular}{llcccccc}
\toprule

Model  & Method & Init weight & Epoch & LR init & Weight decay & $\rho$ & \# samples \\ \midrule \midrule
\multicolumn{8}{c}{\textbf{CIFAR-100}} \\
PreResNet-164    & SWAG  &   &  &   &  &  & 30 \\
    & MC-Drop & Scratch  & 300  &  0.1 & 3e-4 & 0.1 & 30 \\
    & Deep-Ens &   &  &   &  &  & 3 \\  \cmidrule{1-8}
WideResNet28x10    & SWAG  &   &   &  & &  & 30 \\
    & MC-Drop & Scratch  & 300  &  0.1 & 5e-4 & 0.1 & 30 \\
    & Deep-Ens &   &   &   &  &  & 3 \\ \cmidrule{1-8}
Resnet10    & SGVB &  &  &  &  & 5e-3 &  \\
    &F-SGVB + Geometry &   &   &   &  & 5e-4 &  \\
    & SGVB-LRT & Scratch  & 200  &  0.001 & 5e-4 & 5e-3 & 30 \\  
    &F-SGVB-LRT + Geometry &   &   &   &  & 5e-4 & \\ \cmidrule{1-8}
Resnet18      & SGVB &  &  &  &  & 5e-3 &  \\
    &F-SGVB + Geometry &   &   &   &  & 5e-4 &  \\
    & SGVB-LRT & Scratch  & 200  &  0.001 & 5e-4 & 5e-3 & 30 \\  
    &F-SGVB-LRT + Geometry &   &   &   &  & 5e-4 & \\ \cmidrule{2-8}
    & SWAG & Scratch  & 300  &  0.1 & 5e-4 & 0.1 & 30 \\ \midrule

\multicolumn{8}{c}{\textbf{CIFAR-10}} \\
PreResNet-164    & SWAG  &   &  &   &  &  & 30 \\
    & MC-Drop & Scratch  & 300  &  0.1 & 3e-4 & 0.05 & 30 \\
    & Deep-Ens &  &   &   &  &  & 3 \\  \cmidrule{1-8}
WideResNet28x10    & SWAG  &   &   &  & &  & 30 \\
    & MC-Drop & Scratch  & 300  &  0.1 & 5e-4 & 0.05 & 30 \\
    & Deep-Ens &   &   &   &  &  & 3 \\ \midrule
Resnet10    & SGVB &  &  &  &  & 5e-3 &  \\
    &F-SGVB + Geometry &   &   &   &  & 5e-4 &  \\
    & SGVB-LRT & Scratch  & 200  &  0.001 & 5e-4 & 5e-3 & 30 \\  
    &F-SGVB-LRT + Geometry &   &   &   &  & 5e-4 & \\ \cmidrule{1-8}
Resnet18      & SGVB &  &  &  &  & 5e-3 &  \\
    &F-SGVB + Geometry &   &   &   &  & 5e-4 &  \\
    & SGVB-LRT & Scratch  & 200  &  0.001 & 5e-4 & 5e-3 & 30 \\  
    &F-SGVB-LRT + Geometry &   &   &   &  & 5e-4 & \\ \cmidrule{2-8}
    & SWAG & Scratch  & 300  &  0.1 & 5e-4 & 0.1 & 30 \\ \midrule

\multicolumn{8}{c}{\textbf{ImageNet}} \\
DenseNet-161   & All methods  &  Pre-trained & 10 & 0.001  & 1e-4 & 0.05  & 30 \\
ResNet-152    & All methods  &  Pre-trained & 10 & 0.001  & 1e-4 & 0.05  & 30 \\
\bottomrule
\end{tabular}}
\end{table}


\end{document}